\documentclass{article}
\usepackage{amsmath}

\def\eps{{\epsilon}}

\usepackage{wrapfig}
\usepackage{enumitem}

\DeclareMathAlphabet{\mathsfit}{\encodingdefault}{\sfdefault}{m}{sl}
\SetMathAlphabet{\mathsfit}{bold}{\encodingdefault}{\sfdefault}{bx}{n}

\def\gA{{\mathcal{A}}}
\def\gB{{\mathcal{B}}}

\def\gD{{\mathcal{D}}}
\def\gE{{\mathcal{E}}}
\def\gF{{\mathcal{F}}}
\def\gG{{\mathcal{G}}}
\def\gH{{\mathcal{H}}}

\def\gM{{\mathcal{M}}}

\def\gO{{\mathcal{O}}}

\def\gS{{\mathcal{S}}}

\def\gU{{\mathcal{U}}}

\def\gX{{\mathcal{X}}}
\def\gY{{\mathcal{Y}}}
\def\gZ{{\mathcal{Z}}}

\def\sE{{\mathbb{E}}}

\def\sN{{\mathbb{N}}}

\def\sR{{\mathbb{R}}}

\def\sV{{\mathbb{V}}}

\def\unif{{\mathrm{Uniform}}}

\newcommand{\pr}{\mathrm{Pr}}
\newcommand{\subopt}{\mathrm{SubOpt}}

\newcommand{\Ber}{\textrm{Ber}}
\newcommand{\kl}{\textrm{KL}}

\newcommand{\pdim}{\textrm{Pdim}}
\DeclareMathOperator*{\argmax}{arg\,max}
\DeclareMathOperator*{\argmin}{arg\,min}

\usepackage{xcolor}
\usepackage{environ}
\usepackage{xstring}

\newcommand{\keepcomment}{0}%
\newcommand{\keeprevise}{0}

\definecolor{light-gray}{gray}{0.5}

\newcommand{\thanh}[1]{
    \IfEqCase{\keepcomment}{
        {1}{\textcolor{blue}{\emph{[\textbf{thanh}: #1]}}}
        {0}{}
    }
}
\newcommand{\raman}[1]{
    \IfEqCase{\keepcomment}{
        {1}{\textcolor{red}{\emph{[\textbf{Raman}: #1]}}}
        {0}{}
    }
}

\newcommand{\comment}[1]{
\IfEqCase{\keepcomment}{
        {1}{\textcolor{red}{#1}}
        {0}{}
    }
}

\newcommand{\revise}[1]{
    \IfEqCase{\keeprevise}{
        {1}{\textcolor{orange}{#1 }}
        {0}{#1 }
    }
}

\newcommand{\erase}[1]{
    \IfEqCase{\keeprevise}{
        {1}{\textcolor{red}{\st{#1} }}
        {0}{}
    }
}

\NewEnviron{commentblock}{%
  \color{light-gray}
  \ifnum\keepcomment=1
    \BODY
  \else
  \fi
}

\usepackage{microtype}
\usepackage{graphicx}
\usepackage{subfigure}
\usepackage{booktabs} %
\usepackage{bbm}
\usepackage{hyperref}

\usepackage[accepted]{icml2024}

\usepackage{amsmath}
\usepackage{amssymb}
\usepackage{mathtools}
\usepackage{amsthm}
\usepackage{soul}
\usepackage[capitalize,noabbrev]{cleveref}

\theoremstyle{plain}
\newtheorem{theorem}{Theorem}[section]
\newtheorem{proposition}[theorem]{Proposition}
\newtheorem{lemma}[theorem]{Lemma}
\newtheorem{example}[theorem]{Example}

\theoremstyle{definition}
\newtheorem{definition}[theorem]{Definition}
\newtheorem{assumption}[theorem]{Assumption}
\theoremstyle{remark}
\newtheorem{remark}[theorem]{Remark}

\usepackage[textsize=tiny]{todonotes}

\icmltitlerunning{On The Statistical Complexity of Offline Decision-Making}

\begin{document}

\twocolumn[
\icmltitle{On The Statistical Complexity of Offline Decision-Making}

\begin{icmlauthorlist}
\icmlauthor{Thanh Nguyen-Tang}{yyy}
\icmlauthor{Raman Arora}{yyy}
\end{icmlauthorlist}

\icmlaffiliation{yyy}{Department of Computer Science, Johns Hopkins University, Baltimore 21218, USA}

\icmlcorrespondingauthor{Thanh Nguyen-Tang}{nguyent@cs.jhu.edu}

\icmlkeywords{Machine Learning, ICML}

\vskip 0.3in
]

\printAffiliationsAndNotice{\icmlEqualContribution} %

\begin{abstract}

We study the statistical complexity of offline decision-making with function approximation, establishing (near) minimax-optimal rates for stochastic contextual bandits and Markov decision processes. The performance limits are captured by the pseudo-dimension of the (value) function class and a new characterization of the behavior policy that \emph{strictly} subsumes all the previous notions of data coverage in the offline decision-making literature. In addition, we seek to understand the benefits of using offline data in online decision-making and show nearly minimax-optimal rates in a wide range of regimes.

\end{abstract}

\section{Introduction}

Reinforcement learning (RL) has achieved remarkable empirical success in a wide range of challenging tasks, from playing video games at the same level as humans~\citep{mnih2015human}, surpassing champions at the game of Go~\citep{silver2018general}, {to} defeating top-ranked professional players in StarCraft~\citep{vinyals2019grandmaster}. However, many of these systems require extensive online interaction in game-play with other players who are experts at the task or some form of self-play~\cite {li2016deep, ouyang2022training}. Such online interaction may not be affordable in many real-world scenarios due to concerns about cost,  safety, and ethics (e.g., healthcare and autonomous driving). 
Even in domains where online interaction is possible (e.g., dialogue systems), we would still prefer to utilize available historical, pre-collected datasets to learn useful decision-making policies efficiently. 
Such an approach would allow leveraging plentiful data, possibly replicating the success that supervised learning has had recently \citep{lecun2015deep}. Offline RL has emerged as an alternative to allow learning from existing datasets and is particularly attractive %
when online interaction is prohibitive~\citep{ernst2005tree, lange2012batch, levine2020offline}.

Nevertheless, learning good policies from offline data presents a unique challenge not present in online decision-making: \ul{distributional shift}. 
In essence, the policy that interacts with the environment and collects data differs from the target policy we aim to learn. 
This challenge becomes more pronounced in real-world problems with large state spaces, where it necessitates function approximation to generalize from observed states to unseen ones.

Representation learning is a basic challenge in machine learning. It is not surprising, then, that function approximation plays a pivotal role in reinforcement learning (RL) problems with large state spaces, mirroring its significance in statistical learning theory \citep{vapnik2013nature}. Empirically, deep RL, which employs neural networks for function approximation, has achieved remarkable success across diverse tasks~\citep{mnih2015human, schulman2017proximal}. The choice of function approximation class determines the inductive bias we inject into learning, e.g., our belief that the learner's environment is relatively simple even though the state space may be large.

It is natural, then, to understand different function approximation classes in terms of a tight characterization of their complexity and learnability. In statistical supervised learning, specific combinatorial properties of the function class are known to completely characterize sample-efficient supervised learning in both realizable and agnostic settings~\citep{vapnik1971uniform, alon1997scale, attias2023optimal}. For offline RL, a similar characterization is not known. With that as our motivation, we pose the following fundamental question that has largely remained unanswered:
\emph{What is a sufficient and necessary condition for learnability in offline RL with function approximation?} 

We note that given the additional challenge of distribution shift in offline RL, such a characterization would 
depend not only on the properties of the function class but, more importantly, on the quality of the offline dataset. 
The existing literature on offline RL provides theoretical understanding only for limited scenarios of distributional shifts. These works capture the quality of offline data via a notion of data coverage. The strongest of these notions is that of uniform coverage, which requires that the behavior policy (aka, the policy used for data collection) 
covers the space of all feasible trajectories with sufficient probability~\citep{DBLP:journals/jmlr/MunosS08,chen2019information}. Recent works consider weaker assumptions, including single-policy concentrability, which only requires coverage for the trajectories of the target policy that we want to compete against~\citep{DBLP:conf/uai/LiuSAB19,rashidinejad2021bridging,yin2021towards} and relative condition numbers~\citep{agarwal2021theory,uehara2021pessimistic}. While the works above do not take into account function approximation in their notion of data coverage, %
Bellman residual ratio~\citep{xie2021bellman}, Bellman error transfer coefficient~\citep{song2022hybrid} and data diversity~\citep{nguyen-tang2023on} 
directly incorporate function approximation into the notion of offline data quality. 
Many of these notions are incompatible in that they give differing views of the landscape of offline learning. Further, there are no known lower bounds establishing the necessity of any of these assumptions. This suggests that there are gaps in our understanding of when offline learning is feasible. 

In this paper, we work towards giving a more comprehensive and tighter characterization of problems that are learnable with offline data. Our key contributions are as follows. 
\begin{itemize}
\item We introduce the notion of policy transfer coefficients, which subsumes other notions of data coverage. 
\item In conjunction with the pseudo-dimension of the function class, policy transfer coefficients give a tight characterization of offline learnability. Specifically, the class of offline learning problems characterized as learnable by policy transfer coefficients subsumes the problems characterized as learnable in prior literature. We provide (nearly) matching minimax lower and upper bounds for offline learning. %

\item Our results encompass offline learning in the setting of multi-armed bandits, contextual bandits, and Markov decision processes. We consider a variety of function approximation classes, including linear, neural-network-based, and, more generally, any function class with bounded $L_1$ covering numbers. 
\item We extend our results to the setting of hybrid offline-online learning and formally characterize the value of offline data in online learning problems. 
\item We overcome various technical challenges such as giving the uniform Bernstein's inequality for Bellman-like loss using empirical $L_1$ covering numbers (see \Cref{section: uniform Bernstein's inequality} and \Cref{remark: compare uniform Bernstein's inequality to Krish's uniform Freedman-type inequality})
and removing the blowup of the number of iterations of the Hedge algorithm (see \Cref{remark: negligible hedging cost}). These may be of independent interest in themselves. 
\end{itemize}

    The rest of the paper is organized as follows. In~\Cref{chapter: background}, we introduce a formal setup for offline decision-making problems, where we focus mainly on the contextual bandit model to avoid deviating from the main points. In~\Cref{section: offline DM as transfer learning}, we introduce policy transfer coefficients, a new notion of data coverage. In \Cref{section: lower bounds for offline CB} and \Cref{section: upper bound ofr CB offline}, we provide lower bounds and upper bounds, respectively,  for offline decision-making (in the contextual bandit model). In~\Cref{section: extension to hybrid setting}, we consider a hybrid offline-online setting. We extend our results to Markov decision processes (MDPs) in~\Cref{section: extension to MDPs}. We conclude with a discussion in \Cref{section: discussion}.

\section{Background and Problem Formulation}
\label{chapter: background}

\subsection{Stochastic contextual bandits} 
We represent a stochastic contextual bandit environment with a tuple  $(\gX, \gA, \gD)$, where $\gX$ denotes the set of contexts, $\gA$ denotes the space of actions and $\gD \in \Delta(\gX \times \gY)$ denotes an unknown joint distribution over the contexts and rewards. Without loss of generality, we take $\gY:=[0,1]^\gA$. 
A learner interacts with the environment as follows. At each time step, the environment samples $(X,Y) \sim \gD$,
the learner is presented with the context $X$, she commits to an action $a \in \gA$, and observes reward $Y(a)$. 

We model the learner as stochastic. The learner maintains a stochastic policy $\pi: \gX \rightarrow \Delta(\gA)$, i.e., a map from the context space to a distribution over the action space. The value, $V^{\pi}$, of a policy $\pi$ is defined to be its expected reward, 
\begin{align*}
    V^{\pi}_\gD := \sE_{(X,Y) \sim \gD, A \sim \pi(\cdot|X)}[Y(A)]. 
\end{align*}
The sub-optimality of $\hat{\pi}$ w.r.t. any policy $\pi$ is defined as:
\begin{align}
    \subopt_{\gD}^{\pi}(\hat{\pi}) = V_{\gD}^{\pi} - V_{\gD}^{\hat{\pi}}. 
    \label{eq: suboptimality}
\end{align}
We often suppress the subscript in $V^{\pi}_\gD$ and  $\subopt_{\gD}^{\pi}(\hat{\pi})$. %

\subsection{Offline data} Let $S = \{(x_i, a_i, r_i)\}_{i \in [n]}$ be a dataset collected by a (fixed, but unknown) ``behavior'' policy $\mu$, i.e., $(x_i, y_i) \overset{i.i.d.}{\sim} \gD$, $a_i \sim \mu(\cdot|x_i)$, and $r_i = y_i(a_i)$ for all $i \in [n]$. The goal of offline learning is to learn a policy $\hat{\pi}$ from the offline data such that it has small sub-optimality $\subopt_{\gD}^{\pi}(\hat{\pi})$ for as wide as possible a range of comparator policies $\pi$ (possibly including an optimal policy $\pi_{\gD}^* \in \argmax_{\pi} V^{\pi}_{\gD}$). 

\subsection{Function approximation} A central aspect of any value-based method for a sequential decision-making problem is to employ a certain function class {$\gF \subset [0,1]^{\gX \times \gA }$} for modeling rewards in terms of contexts and actions; it is typical to solve a regression problem using squared loss. 
The choice of the function class reflects learner's inductive bias or prior knowledge about the task at hand. In particular, we often make the following realizability assumption~\citep{chu2011contextual,agarwal2012contextual,foster2020beyond}.

\begin{assumption}[Realizability] There exists an $f^* \in \gF$ such that $f^*(x,a) = \sE[Y(a) | X = x], \forall (x,a) \in (\gX \times \gA)$. 
  \label{assumption: realizability cb}
\end{assumption}  

We will utilize a well-understood characterization of the complexity of real-valued function classes from statistical learning, that of pseudo-dimension~\citep{pollard1984convergence}. 

\begin{definition}[Pseudo-dimension]
    A set $\{z_1, \ldots, z_d\} \subset \gZ$ is said to be shattered by $\gH \subset \sR^{\gZ}$ if 
there exists a vector $r \in \mathbb{R}^d$ such that for all $\epsilon \in \{\pm 1 \}^m$, there exists $h \in \mathcal{H}$ such that $\textrm{sign}(h(z_i) - r_i)=\epsilon_i$ for all $i \in [n].$ 
The pseudo-dimension $\pdim(\gH)$ is the cardinality of the largest set shattered by $\gH$.  
\end{definition}

{For example, neural networks of depth $L$, with ReLU activation,  and total number of parameters (weights and biases) equal to $W$ have pseudo-dimension of $\gO(W L \log(W))$ \citep{bartlett2019nearly}.}

\begin{assumption} 
Define the function class associated with each fixed action, $\gF(\cdot,a) := \{f(\cdot,a): f \in \gF\}$. We assume that 
    $\sup_{a \in \gA} \pdim \gF(\cdot,a) \leq d$.
    \label{assumption: pseudo-dimension of F for CB}
\end{assumption}
    The reason we assume that the function class associated with each action has a bounded pseudo-dimension is that we can provide (nearly) matching lower and upper bounds for such a function class. We also provide upper bounds in terms of the covering number of a function class.

\begin{definition}[Covering number]
    Let $S = \{z_1, \ldots, z_n\} \subset \gZ$, and $\hat{P}_S(\cdot) = \frac{1}{n} \sum_{i=1}^n \delta_{z_i}(\cdot)$ be the empirical distribution, where $\delta_{z}$ is the Dirac function at $z$. For any $p >0$ and any $\gH \subset \sR^{\gZ}$, let $N_p(\gH, \eps, L_p(\hat{P}_S))$ denote the size of the smallest $\gH'$ such that: 
    \begin{align*}
        \forall h \in \gH, \exists h' \in \gH':  \|h - h'\|_{ L_p(\hat{P}_S)} \leq \eps,
    \end{align*}
    where $\|h - h'\|_{ L_p(\hat{P}_S)} := \left(\frac{1}{n} \sum_{i=1}^n |h(z_i) - h'(z_i)|^p\right)^{1/p}$ is a pseudo-metric on $\gH$. We define the (worst-case) $L_p$ covering number $N_p(\gH, \eps, n)$ as: 
    \begin{align*}
        N_p(\gH, \eps, n) = \sup_{S: |S|=n}N_p(\gH, \eps, L_p(\hat{P}_S)).
    \end{align*}
\end{definition}

\paragraph{Notation.} Let $f(x,\pi) := \sE_{a \sim \pi(\cdot|x)}f(x,a), \forall x$, and $\gF(\cdot, \Pi) := \{f(\cdot, \pi): f \in \gF, \pi \in \Pi \}$, where {$\Pi := \Delta(\gA)^{\gX}$ is the set of all possible (Markovian) policies}. Let $\gD \otimes \pi$ denote the distribution of the random variable $(x,a,r)$, where $(x,y) \sim \gD, a \sim \pi(\cdot|x), r=y(a)$. Let $l_f$ denote the random variable $(f(x,a) -r)^2$ where $(x,a,r) \sim \gD \otimes \mu$. The empirical and the population means of $l_f$ are represented as 
$\hat{P} l_f = \frac{1}{n} \sum_{i=1}^n l_f(x_i, a_i, r_i)$ and $P l_f = \sE_{(x,a,r) \sim \gD \otimes \mu} [l_f(x,a,r)]$, respectively.
The empirical and the population means of $f$ under policy $\pi$ are denoted as  $\hat{P} f(\cdot, \pi) = \frac{1}{n} \sum_{i=1}^n f(x_i, \pi)$, and $P f(\cdot, \pi) := \sE_{x \sim \gD}[f(x,\pi)]$, respectively. 
We write $a \lesssim b$ to mean $a = \gO(b)$, suppressing only absolute constants, and $a ~\widetilde{<}~ b$ to mean $a = \widetilde{\gO}(b)$, further suppressing log factors.  Define 
$[x]_1 := \max\{\sqrt{x}, x\}$. 

\section{Offline Decision-Making as Transfer Learning}
\label{section: offline DM as transfer learning}
We view offline decision-making as transfer learning where the goal is to 
utilize pre-collected experiences for learning new tasks. 
A key observation we leverage is that there are parallels in how the two areas capture distribution shift -- a common challenge in both settings. Transfer learning uses various notions of distributional discrepancies \citep{ben2010theory,david2010impossibility,germain2013pac,sugiyama2012density,mansour2012multiple,tripuraneni2020theory,watkins2023optimistic} to capture distribution shift between the source tasks and the target task, much like how offline decision-making uses various notions of data coverage to measure the distributional mismatch due to offline data. We consider a new notion of data coverage inspired by transfer learning, which will be shown shortly to tightly capture the statistical complexity of offline decision-making from a behavior policy.

\begin{definition}[Policy transfer coefficients]
Given any policy $\pi$, $\rho \geq 0$ is said to be a \ul{policy transfer exponent} from $\mu$ to $\pi$ w.r.t. $(\gD, \gF)$ if there exists a finite constant $C$, called \ul{policy transfer factor}, such that:
    \begin{align}
        \forall f \in \gF: \left( \sE_{\gD \otimes \pi}[f^* - f] \right) ^{2 \rho} \leq C \sE_{\gD \otimes \mu}[(f^* - f)^2].
        \label{eq: transfer exponent}
    \end{align}
    Any such pair $(\rho, C)$ is said to be a \ul{policy transfer coefficient} from $\mu$ to $\pi$ w.r.t. $(\gD, \gF)$.
We denote the \emph{minimal} policy transfer %
exponent by $\rho_\pi$.\footnote{If $\rho$ is a policy transfer exponent, so is any $\rho' \geq \rho$.} The \emph{minimal} policy transfer factor corresponding to $\rho_{\pi}$ is denoted as $C_\pi$.

    \label{definition: transfer exponent}
\end{definition}

\begin{remark}
    Our definition of policy transfer resembles and is directly inspired by the notion of transfer exponent by \citep{hanneke2019value}, which we refer to as  Hanneke-Kpotufe (HK) transfer exponent for distinction. A direct adaptation of the HK transfer exponent would result in: 
      \begin{align}
        \forall f \in \gF:  \sE_{\gD \otimes \pi}[(f^* - f)^2] ^{\rho} \leq C \sE_{\gD \otimes \mu}[(f^* - f)^2]. 
        \label{eq: HK transfer exponent}
    \end{align}
    Note the difference in the LHS of \Cref{eq: transfer exponent} and that of \Cref{eq: HK transfer exponent}. When defining policy transfer coefficients, we %
    use $\ell_2$-distance\footnote{The %
    {$\ell_2$-distance} from $f$ to $f^*$ corresponds to the excess risk of $f$ w.r.t. squared loss when assuming realizability.}  $ \sE_{\gD \otimes \mu}[(f^* - f)^2]$ w.r.t. the behavior policy to control the \emph{expected value gap} $\sE_{\gD \otimes \pi}[f^* - f]$, whereas the HK transfer exponent directly requires a bound on squared distance $\sE_{\gD \otimes \pi}[(f^* - f)^2]$ w.r.t. to the target policy.
    While this appears to be a small change, %
    our notion bears a deeper connection with offline decision-making that the HK transfer exponent cannot capture. Perhaps, the best way to demonstrate this is through a concrete example where the policy transfer coefficient tightly captures the learnability of offline decision-making while the HK transfer exponent fails -- we return to this in 
    \Cref{example: halfspace-and-spheres}. 
    For now, it suffices to argue the tightness of our notion more generally. 
    Indeed, our policy transfer exponent notion allows a tighter characterization of the transferability for offline decision-making, as indicated by Jensen's inequality: 
    \begin{align*}
        \left( \sE_{\gD \otimes \pi}[f^* - f] \right)^2 \leq \sE_{\gD \otimes \pi}[(f^* - f)^2].
    \end{align*}
     Now, consider a random variable, $\zeta$, that is distributed according to the Bernoulli distribution $\textrm{Ber}(p)$. Then, we have $\sE[\zeta^2]/ |\sE[\zeta]|^2 = 1/p$. This means that in this example, the transfer factor $C_1$ (corresponding to $\rho=1$) in our definition is smaller than the transfer factor implied by the original definition of \citep{hanneke2019value} by a factor of $p$, which is significant for small values of $p$. %

\end{remark}
\subsection{Relations with other notions of data coverage}
In this section, we highlight the properties of transfer exponents and compare them with other notions of data coverage considered in offline decision-making literature. Perhaps the most common notions of data coverage are that of single-policy concentrability coefficients~\citep{DBLP:conf/uai/LiuSAB19,rashidinejad2021bridging}, relative condition numbers (for linear function classes) \citep{agarwal2021theory,uehara2021pessimistic}, and data diversity \citep{nguyen-tang2023on}.\footnote{Data diversity of \citet{nguyen-tang2023on}, in fact,  generalizes Bellman error transfer coefficient of~\citet{song2022hybrid} to allow an additive error.} We demonstrate that policy transfer coefficients strictly generalize all these prior notions, in the sense that bounds on the prior notions of data coverage always imply bounds on transfer exponents but not vice versa. Specifically, there are problem instances for which the existing measures of data coverage tend to infinity, yet these problems are learnable given the characterization in terms of transfer coefficients.

\paragraph{Compared with concentrability coefficients.} The concentrability coefficient between $\pi$ and $\mu$ %
is defined as $\kappa_\pi := \sup_{x,a} \frac{\pi(a|x)}{\mu(a|x)}$. The finiteness of $\kappa_\pi$ is widely used as one of the sufficient conditions for sample-efficient offline decision-making. By definition, the policy transfer factor corresponding to the policy transfer exponent of $1$, is always upper-bounded by $\kappa_\pi$. The finiteness of $\kappa_\pi$ requires the support of $\mu$ to contain that of $\pi$. However, offline decision-making does not even need overlapping support between a target policy and the behavior policy. 
\begin{example}
Consider $|\gX| = 1$ and $\gA = \sR^d$. Let $\gF$ be the class of $d$-dimensional halfspaces that pass through the origin, and $\Pi$ be the set of $d$-dimensional spheres centered at the origin, $\Pi = \{\textrm{Uniform}(\{a \in \sR^d: \|a\|=r\}): r \geq 0\}$ (see the figure below). 
 \begin{figure}[h!]
     \centering
     \includegraphics[scale=0.4]{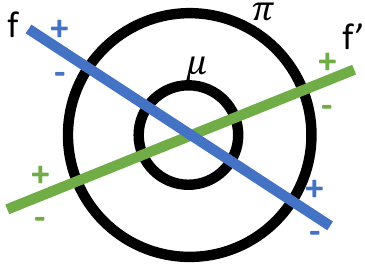}
 \end{figure}
\label{example: halfspace-and-spheres}
\end{example}

 \Cref{example: halfspace-and-spheres}, taken from \cite{hanneke2019value}, gracefully reflects the ``correctness'' of policy transfer exponent at characterizing learnability in offline settings, whereas both the HK transfer exponent and the concentrability coefficient fail to do so. In this case, by direct computation we have that $\sE_{\gD \otimes \pi}[f^* - f] = 0$, thus the minimal policy transfer exponent from any behavior policy $\mu \in \Pi$ to any different policy $\pi \in \Pi$ approaches zero, i.e., $\rho_\pi \rightarrow 0$, while the HK transfer exponent, denoted $\rho_{\pi}^{\textrm{HK}}$, is $1$ and the concentrability coefficient $\kappa_\pi$ is infinity. These different values of data quality measures translate into different predictive bounds of offline learning for \Cref{example: halfspace-and-spheres}. For concreteness, we summarize these bounds in \Cref{tab: predictive bounds for the halfspace example}.
 \begin{table}[h!]
     \centering
     \begin{tabular}{|c|c|}
     \hline
    \textbf{Data quality measure}  & \textbf{Predictive bound} \\
    \hline 
    \hline
       Policy transfer exponent $\rho_{\pi} \rightarrow 0$ & $0$ \\
       \hline
       HK transfer exponent $\rho_{\pi}^{\textrm{HK}}=1$ & $\sqrt{\frac{C_{\pi}}{n}}$ \\ 
       \hline
       Concentrability coefficient $\kappa_\pi \!\!=\! \infty$ & $\infty$\\
       \hline
     \end{tabular}
     \caption{Predictive bounds for offline learning of \Cref{example: halfspace-and-spheres} stemming from different measures of data quality. The bounds for the policy transfer exponent and the HK transfer exponent are obtained from \Cref{theorem: upper bounds of OfDM-Hedge CB} (note that our results are applicable to HK transfer exponents as well), while the bound for the concentrability coefficient is obtained from \citet{rashidinejad2021bridging}.}
     \label{tab: predictive bounds for the halfspace example}
 \end{table}
 
 Which of the above data quality measures give a tight characterization of learnability of problem in \Cref{example: halfspace-and-spheres}?  With the realizability assumption, the value $V^{\pi}$ of any policy in $\Pi$ equals $1/2$, regardless of the policy. Thus, the true sub-optimality is zero. This is tightly bounded by the predictive bound of our policy transfer exponent, while those by HK transfer exponent and concentrability coefficient give vacuous bounds, as shown in \Cref{tab: predictive bounds for the halfspace example}.

 Another interesting remark is that for offline decision-making, it is not always necessary to learn the true reward function. This task requires the sample complexity of $\tilde{\Theta}(\frac{d}{\eps})$ -- this is captured by the HK transfer exponent at the value of $1$; see \citep[Example~1]{hanneke2019value}.

\paragraph{Compared with the data diversity of \citet{nguyen-tang2023on}.} The notion of data diversity, recently proposed by \citet{nguyen-tang2023on} is motivated by the notion of task diversity in transfer learning~\citep{tripuraneni2020theory}. \citet{nguyen-tang2023on} show that data diversity subsumes both the concentrability coefficients and the relative condition numbers and that it can be used to derive strong guarantees for offline decision-making and state-of-the-art bounds when the function class is finite or linear. Their data diversity notion is, in fact, the policy transfer factor corresponding to policy transfer exponent $\rho=1$. The following example, which is modified from \citep[Example~3]{hanneke2019value}, shows that data diversity can be infinite while the minimal transfer exponent is finite.

\begin{example}
  Consider $|\gX| = 1, \gA = [-1,1], \gF = \{f_t: t \in [-1,1]\}$ where $f_t(a) = \mathbbm{1}\{a \geq t\}$ is \erase{a step function} \revise{a $1$-dimensional threshold}, and let $f_0$ be the optimal mean reward function. Let $\iota > 0$ be any positive scalar. Consider the following distribution: $\mu(a) \propto a^{2\iota-1}$ for $a \geq 0$ and $\mu(a)$ is uniform for $ a \in [-1,0]$. Let $\pi$ be the uniform distribution over $\gA$. By direct computation, for any $t \geq 0$, $|\sE_{\pi}[f_0 - f_t]| \propto t$ and $\sE_{\mu}[(f_0-f)^2] \propto t^{2\iota}$. Thus, no $\rho < 2\iota$ can be a policy transfer exponent, as $\lim_{t \rightarrow 0} |\sE_{\pi}[f_0 - f_t]|^{\rho} / \sE_{\mu}[(f_0-f)^2] \rightarrow \infty$. 
  Thus, the bound in \citet{nguyen-tang2023on} becomes vacuous. However, as long as the squared Bellman error $\sE_{\mu}[(f_0-f_t)^2]$ can predict the Bellman error $|\sE_{\pi}[f_0 - f_t]|$, one should expect that offline learning is still possible, albeit at a slower rate. 
  \label{example: minimal policy transfer exponent is chosen}
\end{example}

As our policy transfer coefficients cover the data diversity of \citet{nguyen-tang2023on} as a special case ($\rho=2$), which, in turn, is a generalization of the relative condition number, we refer to \citet{nguyen-tang2023on} for a detailed comparison with the relative condition number. 

\section{Lower Bounds}
\label{section: lower bounds for offline CB}
Let $\gB(\rho, C, d)$ denote the class of offline learning problem instances with any distribution $\gD$ over contexts and rewards, any function class $\gF$ that satisfies Assumptions~\ref{assumption: realizability cb} and \ref{assumption: pseudo-dimension of F for CB}, 
a behavior policy $\mu$, and all policies $\pi \in \Pi$ such that policy transfer coefficients w.r.t. $\mu$ are $(\rho, C)$. For this class, we give a lower bound on the sub-optimality of any offline learning algorithm.

\begin{theorem}
       For any $C > 0, \rho \geq 1, n \geq d \cdot \max\{2^{2 \rho - 4} C,  C^{\frac{1}{\rho-1}}/32 \}$, we have
    \begin{align*}
        \inf_{\hat{\pi}(\cdot)} \sup_{(\gD, \mu, \pi, \gF) \in \gB(\rho,C,d)} \sE_{\gD} \left[ \subopt_\gD^{\pi}(\hat{\pi}) \right] \gtrsim \left(\frac{Cd}{n} \right)^{\frac{1}{2 \rho}},
    \end{align*}
    \revise{where the infimum is taken over all offline algorithm $\hat{\pi}(\cdot)$ (a randomized mapping from the offline data to a policy).}
    \label{theorem: lower bound for offline CB}
\end{theorem}

The lower bound in \Cref{theorem: lower bound for offline CB} is information-theoretic, i.e., it applies to \emph{any} algorithm for problem class $\gB(\rho, C, d)$. The lower bound is obtained 
by constructing a set of hard contextual bandit (CB) instances $\{\gD_i\}$ that are supported on $d$ data points. Then, for each $\gD_i$, we pick the hardest comparator policy $\pi = \pi^*_{\gD_i}$ and design a behavior policy that satisfies the policy transfer condition. We pick a simple enough function class that satisfies realizability and ensures that the policy transfer exponents and the pseudo-dimension are bounded. 
 We then proceed to show that given a behavior policy $\mu$, for any two CB instances $\gD_i$ and $\gD_j$ that are close to each other (i.e., $\textrm{KL}[(\gD_i \otimes \mu)^n \| (\gD_j \otimes \mu)^n]$ is small)  the corresponding optimal policies disagree. A complete proof is given in \Cref{section: proofs of lower bounds for offline CB}.

\section{Upper Bounds}
\label{section: upper bound ofr CB offline}
Next, we show that there exists an offline learning algorithm that is agnostic to the minimal policy transfer coefficient of any policy, yet it can compete uniformly with all comparator policies, as long as their minimal policy transfer exponent is finite. For this algorithm, we give an upper bound for VC-type classes that matches the lower bound in the previous section up to log factors, ignoring the dependence on $K = |\gA|$. For more general function classes, we only provide upper bounds. %

The general recipe for our algorithm (\Cref{algorithm: OfDM-Hedge}) is rather standard. 
We follow the actor-critic framework for offline RL studied in several prior works \citep{zanette2021provable,xie2021bellman,nguyen-tang2023on}. The algorithm alternates between computing a pessimistic estimate of the actor and improving the actor with the celebrated Hedge algorithm~\citep{freund1997decision}.

\begin{algorithm}
   \caption{Hedge for Offline Decision-Making (OfDM-Hedge)}
\begin{algorithmic}[1]
   \STATE {\bfseries Input:} Offline data $S$, function class $\gF$
   \STATE {\bfseries Hyperparameters:} Confidence parameter $\beta$, learning rate $\eta$, number of iterations $T$
   \STATE Initialize $\pi_1(\cdot|x) = \textrm{Uniform}(\gA)$, $\forall x \in \gX$
   \FOR{$t=1$ {\bfseries to} $T$}
   \STATE Pessimism: $f_t = \displaystyle\argmin_{f \in \gF: \hat{P} l_f - \hat{P} l_{\hat{f}} \leq \beta} \hat{P} f(\cdot, \pi_t)$
   \STATE Hedge: $\pi_{t+1}(a|x) \propto \pi_t(a|x) e^{\eta f_t(x,a)}, \forall (x,a)$
   \ENDFOR
   \STATE {\bfseries Output:} A randomized policy $\hat{\pi}$ as a uniform distribution over $\{\pi_t\}_{t \in [T]}$.
   \label{algorithm: OfDM-Hedge}
\end{algorithmic}
\end{algorithm}

The following result bounds the suboptimality of  OfDM-Hedge up to absolute constants which  we ignore for ease of exposition (see \Cref{theorem: theorem: upper bounds of OfDM-Hedge CB with exact constants} for exact constants). 
\begin{theorem}
    Fix any $\delta \in [0,1], \eps \geq 0$. Assume that $|\gA| = K$. Then, for any $(\gD,\gF)$ such that \Cref{assumption: realizability cb} holds, invoking  \Cref{algorithm: OfDM-Hedge} with $\beta = \Theta \left(\eps + \frac{ \log (N_1(\gF, \eps,n)/\delta)}{n} \right)$ and $\eta = \Theta(\sqrt{\frac{\log K}{T}})$ returns a policy $\hat{\pi}$ such that with probability at least $1 - \delta$,   $\forall \pi \in \Pi$, $\forall T \in \sN$, %
    \begin{eqnarray}
\sE[\subopt_{\gD}^\pi(\hat{\pi}) | S] \lesssim C_\pi^{\frac{1}{2 \rho_\pi}} \left(\! \eps + \frac{\log (N_1(\gF, \eps,n)/\delta)}{n} \right)^{\frac{1}{2 \rho_{\pi}}} \nonumber \\
   \hspace*{-30pt} \!\!\!\!\!\!\!\!\!\!\!\!\! + \mathbbm{1}_{\{|\gX| > 1\}} \left\{   \left[\frac{\log (N_1(\gF(\cdot, \Pi), \eps,n)/\delta) }{n} \right]_{1} \!\!\!+ \epsilon \right\} +\sqrt{\frac{\log K}{T}}. \nonumber
    \end{eqnarray}
    \label{theorem: upper bounds of OfDM-Hedge CB}
\end{theorem}

The upper bound consists of three main terms, which represent the transfer cost for offline decision-making, standard statistical learning error, and the optimization error of the Hedge algorithm, respectively. Note that \Cref{theorem: upper bounds of OfDM-Hedge CB} does not require \Cref{assumption: pseudo-dimension of F for CB}. 
\begin{remark}[\textbf{Adaptive to policy transfer coefficients}]
    OfDM-Hedge does not need to know the policy transfer coefficients of any target policy, yet it is adaptively competing with any target policy $\pi$ characterized by minimal policy transfer exponent $\rho_\pi$ and %
    policy transfer factor $C_\pi$. 
\end{remark}

\begin{remark}[\textbf{No cost blowup of the Hedge 
algorithm}]
    Note that the first two terms in the upper bound in \Cref{theorem: upper bounds of OfDM-Hedge CB} are completely agnostic to the number of iterations $T$. Thus, scaling up $T$ to drive the last term to zero does not increase the effective dimension -- a desirable property that is absent in the bounds of similar algorithms provided in \citet{zanette2021provable,xie2021bellman,nguyen-tang2023on}. 
    This difference stems from the fact that prior work uses uniform convergence over all policies that can be generated by the Hedge algorithm over $T$ steps. We, however, recognize that any policy $\pi_t$ that is generated by a step of the Hedge algorithm is only used in one single form: $f(\cdot, \pi_t)$ for $f \in \gF$. Thus, it suffices to control the complexity of $\gF(\cdot, \Pi)$ where $\Pi$ is the set of all possible stationary policies. This elegant trick becomes more clear in terms of benefit in comparison with other works that suffer sub-optimal rates due to this cost blowup of the Hedge algorithm; we refer the reader to Section~\ref{section:MDP-upper-bound} for further discussion. 
    
    \label{remark: negligible hedging cost}
\end{remark}

\begin{remark}[\textbf{Potential barrier for offline decision-making in the \underline{fast transfer regime} $\rho_\pi < 1$}] The second term in the bound above vanishes when $|\gX| = 1$ (i.e., in a %
multi-armed bandit setting). This yields error rates that are faster than $1/\sqrt{n}$ if the rate of transfer from $\mu$ to $\pi$ is fast (i.e., $\rho_\pi < 1$). 
One such example is \Cref{example: halfspace-and-spheres} where $\rho \rightarrow 0$, which our upper bounds capture rather precisely. In the general case with $|\gX| > 1$, 
OfDM-Hedge is unable to take advantage of fast policy transfer exponent $\rho_\pi < 1$, as the bound also involves
the second term stemming from standard statistical learning over the context domain $\gX$. This is quite different from the supervised learning setting of \citet{hanneke2019value} where we can always ensure a fast transfer rate when $\rho_\pi < 1$. 
Further, note that our lower bound (\Cref{theorem: lower bound for offline CB})  excludes the fast transfer regime $\rho_\pi < 1$. So, it remains unclear if, any algorithm in the offline setting can leverage fast transfer. 

\end{remark}
Next, we instantiate the upper bound in \Cref{theorem: upper bounds of OfDM-Hedge CB} for VC-type function classes. 

\begin{example}[VC-type / parametric classes]
Under the same setting as \Cref{theorem: upper bounds of OfDM-Hedge CB}, if \Cref{assumption: pseudo-dimension of F for CB} holds, the suboptimality of OfDM-Hedge is bounded by 
(ignoring $\log(1/\delta)$ terms): 
\begin{align*}
      \max \left\{ \left[\frac{K d\log(dn)}{n} \right]_1, \left(\frac{C_\pi K d\log(dn)}{n} \right)^{\frac{1}{2 \rho_\pi}} \right\}.  
\end{align*}
\label{example: parametric class}
\end{example}
    The VC-type parametric classes include the $d$-dimensional linear class as a special case. The bound above follows from \Cref{theorem: upper bounds of OfDM-Hedge CB} and using the following inequalities: $ \max_{a \in \gA} N_1(\gF(\cdot,a), \eps, n) \leq e (d+1) (\frac{2e}{\eps})^d$ \citep{haussler1995sphere} and $ \max \left\{N_1(\gF(\cdot, \Pi), \eps,n), N_1(\gF, \eps,n)\right\} \leq \left (\max_{a \in \gA} N_1(\gF(\cdot,a), \eps/K, n) \right)^K$ (\Cref{lemma: combine coverining numbers to those of factorized classes}).

    For the common ``large-sample, difficult-transfer'' regime, i.e., $n \geq Kd \log(dn)$ and $\rho_\pi \geq 1$, the upper bounds for the parametric classes in \Cref{example: parametric class} match the lower bound in \Cref{theorem: lower bound for offline CB}, ignoring constants, log factors and dependence on $K$. 

The upper bound in Theorem~\ref{theorem: upper bounds of OfDM-Hedge CB} also applies to nonparametric classes,  %
though we do not provide a (matching) lower bound for this case. 
We leave that for future work. 
\begin{example}[Nonparametric classes]
Under the same setting as \Cref{theorem: upper bounds of OfDM-Hedge CB}, assume that for all $a$, $\gF(\cdot, a)$ scales polynomially with the inverse of the scale, i.e., for some $p > 0$
    \begin{align*}
        \log N_1(\eps, \gF(\cdot,a), n) \leq \left(\frac{1}{\eps} \right)^p, \forall a \in \gA, \forall \eps \geq 0.
    \end{align*}
    Then, the upper bound in \Cref{theorem: upper bounds of OfDM-Hedge CB} is of the order (ignoring $\log(1/\delta)$ terms):
    \begin{align*}
          \max\left\{C_\pi^{\frac{1}{2 \rho_\pi}} \left(\frac{K}{n} \right)^{\frac{1}{2 \rho_\pi (p+1)}}, \left( \frac{K}{n} \right)^{\frac{1}{1+p}},  \left( \frac{K}{n} \right)^{\frac{1}{2(1+p)}} \right\}.
    \end{align*}
    \label{example: non-parametric classes}
\end{example}
    In the typical ``large-sample difficult-transfer'' regime, the above yields an error rate of $ C_\pi^{\frac{1}{2 \rho_\pi}} \left(\frac{K}{n} \right)^{\frac{1}{2 \rho_\pi(p+1)}}$ which vanishes with $n$. Typical examples of nonparametric classes include infinite-dimension linear classes with $p = 2$ \citep{zhang2002covering}, the class of  $1$-Lipschitz functions over $[0,1]^d$, with $p = d$ \citep{slivkins2011contextual}, the class of Holder-smooth functions of order $\beta$ over  $[0,1]^d$, with $p = d/\beta$ \citep{rigollet2010nonparametric}, and neural networks with spectrally bounded norms, with $p = 2$ \citep{bartlett2017spectrally}. 

\section{Offline Data-assisted Online Decision-Making}
\label{section: extension to hybrid setting}
In this section, we consider a hybrid setting, where in addition to the offline data $S = \{(x_i, a_i, r_i)\}_{i \in [n]}$, the learner is allowed to interact with the environment for $m$ rounds. 
The goal is to output a policy $\hat{\pi}_{\textrm{hyb}}$ with small  $\subopt_{\gD}^{\pi}(\hat{\pi}_{\textrm{hyb}})$ w.r.t. $\pi$ with a high value $V^{\pi}_\gD$. We assume realizability (\Cref{assumption: realizability cb}) and for simplicity, we focus only on VC-type function classes with pseudo-dimension at most $d$. 
Let $\rho^* = \rho_{\pi^*}$ and $C^* = C_{\pi^*}$ denote the policy transfer coefficients for an optimal policy $\pi^*$. 

To avoid deviating from the main point, in this section, we focus on the ``large sample, difficult transfer'' regime, where $d \geq Kd \log(dn)$ and $\rho^* \geq 1$. The key question we ask is whether a learner can perform better in a hybrid setting than in a purely online or offline setting.

\subsection{Lower bounds}
We start with a lower bound for any hybrid learner for the class of problems $\gB(\rho, C, d)$ as in Section~\ref{section: lower bounds for offline CB}.

\begin{theorem}
    For any $C > 0, \rho \geq 1$, and sample size $n \geq d\max\{ 2^{2\rho -4} C ,  \frac{C^{\frac{1}{\rho-1}}}{32} \}$, we have
    \begin{eqnarray}
\inf_{\hat{\pi}_{\textrm{hyb}}(\cdot)}~~\sup_{(\gD, \mu, \pi, \gF) \in \gB(\rho,C,d)} ~\sE_{\gD} \left[ \subopt_\gD^{\pi}(\hat{\pi}_{\textrm{hyb}}) \right] \nonumber\\
     ~~~~\gtrsim~ \min \left\{\left(\frac{C d}{n} \right)^{\frac{1}{2 \rho}}, \sqrt{\frac{d}{m}} \right\}, \nonumber
    \end{eqnarray}
\revise{where the infimum is taken over all possible hybrid algorithm $\hat{\pi}_{\textrm{hyb}}(\cdot)$.}
    \label{theorem: lower bound for the hybrid setting}
\end{theorem}

    Ignoring log factors and dependence on $K$, the first term in the lower bound %
    matches the upper bound for OfDM-Hedge.  Therefore, if that is the dominating term, the learner does not benefit from online interaction. 
    The second term, $\sqrt{d/m}$, matches the upper bound of the state-of-the-art online learner 
    \citep{simchi2022bypassing} with an online interaction budget of $m$. In the regime where the latter term is dominating there is no advantage to having offline data.

To summarize, the lower bound suggests that if the policy transfer coefficient is known a priori to the hybrid learner, 
there is no benefit of mixing the offline data with the online data at least in the worst case. That is, obtaining the nearly minimax-optimal rates for hybrid learning is akin to either discarding the online data and running the best offline learner or ignoring the offline data and running the best online learner. \ul{Which algorithm to run depends on the transfer coefficient.}  That there is no benefit to mixing online and offline data is a phenomenon that has also been discovered in a related setting of policy finetuning~ \citep{xie2021policy}.

\subsection{Upper bounds}
The discussion following the lower bound suggests different algorithmic approaches (purely offline vs. purely online) for different regimes defined in terms of the policy transfer coefficient. However, it is unrealistic to assume that the learner has prior knowledge of the transfer coefficient. We present a hybrid learning algorithm (\Cref{algorithm: hybrid learner}) that offers the best of both worlds. Without requiring the knowledge of the policy transfer coefficient, it produces a policy with nearly optimal minimax rates.

The key algorithmic idea is rather simple and natural. We invoke both an offline policy optimization algorithm and an online learner, resulting in policies $\hat{\pi}_\textrm{off}$ and $\hat{\pi}_\textrm{on}$, respectively. Half of the interaction budget, i.e., $m/2$ rounds, is utilized for learning $\hat{\pi}_\textrm{on}$. For the remaining $m/2$ rounds, we run the EXP4 algorithm~\citep{auer2002nonstochastic} with $\hat{\pi}_\textrm{off}$ and $\hat{\pi}_\textrm{on}$ as expert policies. The output of EXP4 is a uniform distribution over all of the iterates. We denote this randomized policy as $\hat{\pi}_{\textrm{hyb}}.$We can equivalently represent $\hat{\pi}_{\textrm{hyb}}$ as a distribution over $\{\hat{\pi}_\textrm{off}, \hat{\pi}_\textrm{on}$\}.

\begin{algorithm}[h!]
    \begin{algorithmic}[1]
        \STATE {\bfseries Input:} Offline data $S$, function class $\gF$, $m$ online interactions
        \STATE $\hat{\pi}_{\textrm{off}} \leftarrow \textrm{OfflineLearner}(S, \gF)$
        \STATE $\hat{\pi}_{\textrm{on}} \leftarrow \textrm{OnlineLearner}(\frac{m}{2}, \gF)$
        \STATE $\hat{\pi}_{\textrm{hyb}}\! \leftarrow \textrm{EXP4}(\frac{m}{2}, \{\hat{\pi}_{\textrm{off}}, \hat{\pi}_{\textrm{on}}\})$ 
        \label{hybrid: model selection}
        \STATE {\bfseries Output:} $\hat{\pi}_{\textrm{hyb}}$
    \end{algorithmic}
    \caption{Hybrid Learning Algorithm}
    \label{algorithm: hybrid learner}
\end{algorithm}

\begin{proposition}
    \Cref{algorithm: hybrid learner} return a randomized policy $\hat{\pi}_{\textrm{hyb}}$ such that for any $\delta \in (0,1)$, with probability at least $1-\delta$, 
    \begin{align*}
         &\!\!\!\!\!\!\!\max_{{\pi} \in \{\hat{\pi}_{\textrm{off}}, \hat{\pi}_{\textrm{on}}\} }\sE_{\gD}[\subopt_{\gD}^{{\pi}}(\hat{\pi}_{\textrm{hyb}}) | S, S_{\textrm{on}}] \leq \frac{4 \sqrt{2 \log 2}}{\sqrt{m}} \\ 
         &~~~~~~~~~~~~~~~~~~~~~~~~~~~~~~~~~~~~~~+ \frac{32 \log (\log(m/2) / \delta)}{3m} + \frac{4}{m},
    \end{align*}
    where $S$ is the offline data and $S_{\textrm{on}}$ is the online data collected by $\hat{\pi}_{\textrm{on}}$ in \Cref{algorithm: hybrid learner}. 
    \label{proposition: model selection}
\end{proposition}

    The first term in the bound above comes from the guarantees on EXP4 with two experts~\citep[Theorem~18.3]{lattimore2020bandit} and the remaining  terms result from an  (improved) online-to-batch conversion \citep[Lemma~A.5]{nguyentang2022instancedependent} and the assumption that contexts are sampled i.i.d.
    Note that the result above %
    does not require any assumption on $\gF$ and $\gD$. 

Next, we implement  \Cref{algorithm: hybrid learner} using FALCON+ \citep[Algorithm~2]{simchi2022bypassing} for online learning and OfDM-Hedge for offline learning. Then, we have the following bound on the output of the hybrid learning algorithm.

\begin{theorem}
Assume that the closure of $\gF$ is convex. Then,  given that  Assumptions~\ref{assumption: realizability cb} and \ref{assumption: pseudo-dimension of F for CB} hold, we have 
$$
\sE_{\gD}\!\! \left[ \subopt_{\gD}^{\pi^*_\gD}(\hat{\pi}_{\textrm{hyb}}) \right] \!\widetilde{<}\frac{1}{\sqrt{m}} + \min \left\{\!\!\left(\!\frac{C^* K d}{n} \right)^{\frac{1}{2 \rho^*}} \!\!\!\!\!, K\sqrt{\frac{d}{m}} \right\}\!\!. 
$$

\end{theorem}
In the bound above, the first term is the standard error rate of EXP4, the term $\left(\frac{C^* K d}{n} \right)^{\frac{1}{2 \rho^*}}$ is the error rate of OfDM-Hedge (see \Cref{theorem: upper bounds of OfDM-Hedge CB}), and $K\sqrt{\frac{d}{m}}$ is the error rate of FALCON+ (which we tailor to our setting in \Cref{subsection: digest for FALCON+}). Applying \Cref{proposition: model selection} yields the minimum over the error rates of the offline and online learning.   
    
    In the regime that the EXP4 cost is dominated by the second term, \erase{e.g.} \revise{i.e.}, when {$m \geq \left( \frac{n}{C^* K d} \right)^{1/\rho^*}$}, the upper bound on the suboptimality of hybrid learning nearly matches the lower bound, modulo log factors and pesky dependence on $K$. 
    \label{remark: regimes for hybrid learning}

\subsection{Related works for the hybrid setting}
\citet{xie2021policy} consider a related setting of policy finetuning, where the online learner is given an additional reference policy $\mu$ that is close to the optimal policy and the learner can collect data using $\mu$ at any point. They do not consider function approximation and utilize single-policy concentrability coefficients. \citet{song2022hybrid} consider the same hybrid setting as ours (albeit for RL) and show that offline data of good quality can help avoid a need to  explore, thereby resulting in an oracle-efficient algorithm which in general is not possible. 
The statistical complexity they establish for their algorithm is not minimax-optimal, and can get arbitrarily worse with the quality of the offline data. Our guarantees, on the other hand, are adaptive as they revert to online learning  guarantees when the quality of the offline data is low. \citet{wagenmaker2023leveraging} give instance-dependent bounds for hybrid RL with linear function approximation under a uniform coverage condition.  
Other related works include hybrid RL in tabular MDPs~\citep{li2023reward}, a Bayesian framework for incorporating offline data into the prior distribution~\citep{tang2023efficient}, and one-shot online learning using offline data~\citep{zhang2023policy}. 

\section{Offline Decision-Making in MDPs}
\label{section: extension to MDPs}
In this section, we extend our results to offline decision-making in Markov decision processes (MDPs). We show that the key insights developed for the contextual bandit model extend naturally to offline learning of MDPs as we establish nearly matching upper and lower bounds. 

\paragraph{Setup.}
Let $M = \textrm{MDP}(\gX, \gA, [H], \{P_h\}_{h \in [H]}, \{r_h\}_{h \in [H]})$ denote an episodic Markov decision process with state space $\gX$, action space $\gA$, horizon length $H$, transition kernel $P_h: \gX \times \gA \rightarrow \Delta(\gX)$, and mean reward functions $r_h: \gX \times \gA \rightarrow [0,1]$. For any policy $\pi = (\pi_1, \ldots, \pi_H)$ where $\pi_h: \gX \rightarrow \Delta(\gA)$, let $\{Q_h^{\pi}\}_{h \in [H]}$ and $\{V_h^{\pi}\}_{h \in [H]}$ denote the action-value functions and the state-value functions, respectively. For any policy $\pi$, define the Bellman operator $[T^\pi_h g](x,a) := \sE_{x' \sim P_h(\cdot|x,a)} \left[r_h(x,a) + g(x') \right]$. For simplicity, we assume that the MDP starts in the same initial state at every episode. 
Here $\sE_{\pi}[\cdot]$ denotes the expectation with respect to the randomness of the trajectory $(x_h,a_h, \ldots, x_H, a_H)$, with $a_i \sim \pi_i(\cdot|x_i)$ and $x_{i+1} \sim P_i(\cdot|x_i,a_i)$ for all $i$. We assume that $|r_h| \leq 1, \forall h$.

The learner has access to a dataset $S = \{(x_h^{(t)}, a_h^{(t)}, r_h^{(t)})\}_{h \in [H], t \in [n]}$ collected using a behavior policy $\mu$. 
Define the (value) sub-optimality as
  $\subopt^{\pi}_M(\hat{\pi}) := V_1^{\pi}(s_1) - V_1^{\hat{\pi}}(s_1).$ 
Wherever clear, we drop the subscript $M$ in $Q^{\pi}_M$, $V^{\pi}_M$, $d^{\pi}_{M}$, and $ \subopt^{\pi}_M(\hat{\pi})$. 

\paragraph{Function approximation.} We consider a function approximation class $\gF = (\gF_1, \ldots, \gF_H)$ where %
$\gF_h \subseteq [0,H-h+1]^{\gX \times \gA}$.

We make the following standard assumptions regarding how the function class $\gF$ interacts with the underlying MDP $M$. 

\begin{assumption}[Realizability]
$\!\!\forall \pi\! \in \Pi, h\! \in\! [H]$, $Q^{\pi}_h \in \gF_h$. 
\label{assumption: realizability for mdp}
\end{assumption}

\begin{assumption}[Bellman completeness]
   $\forall \pi \in \Pi$, $\forall h \in [H]$, if $f_{h+1} \in \gF_{h+1}$, then $T^{\pi}_h f_{h+1} \in \gF_h$. 
   \label{assumption: Bellman completeness}
\end{assumption}

    For simplicity, we focus on VC-type/parametric $\gF$, though, again, our upper bounds can yield a vanishing error rate for non-parametric classes. %

\begin{assumption}
    $\sup_{h \in [H], a \in \gA}\pdim(\gF_h(\cdot, a)) \leq d$. 
    \label{assumption: parametric class for MDP}
\end{assumption}

\begin{definition}[Policy transfer coefficients]
Given any policy $\pi$, the minimal policy transfer exponent $\rho_\pi$ w.r.t. $(M, \gF, \mu)$ is the smallest $\rho \geq 0$ such that there exists a finite constant $C$ such that for every $h \in [H]$
    \begin{align}
        \forall f_h, g_h \in \gF_h: \left| \sE_{ \pi}[f_h - g_h] \right| ^{2 \rho} \leq C \sE_{\mu}[(f_h - g_h)^2].
        \label{eq: transfer exponent for MDP}
    \end{align}
    The smallest such $C$ w.r.t $\rho_\pi$, denoted $C_\pi$, is called the policy transfer factor. The pair $(\rho_\pi, C_\pi)$ is said to be the policy transfer coefficient of $\pi$. 
    \label{definition: transfer exponent for MDP}
\end{definition}

\subsection{Lower bounds}
Let $\gM(\rho, C, d)$ denote the class of offline learning problem instances with any MDP $M$, any function class $\gF$ that satisfies Assumptions~\ref{assumption: realizability for mdp}, \ref{assumption: Bellman completeness}, and \ref{assumption: parametric class for MDP}, a behavior policy $\mu$, and all policies $\pi \in \Pi$ such that policy transfer coefficients w.r.t. $\mu$ are $(\rho, C)$. For this class, we give a lower bound on the sub-optimality of any offline learning algorithm. 

\begin{theorem} 
For $\rho \geq 1$ and $n \geq \frac{Cd H^{2\rho}}{32}$, we have
    \begin{align*}
        \inf_{\hat{\pi}} \sup_{(M, \gF, \pi, \mu) \in \gM(\rho, C,d)} \sE_M [\subopt^{\pi}_M(\hat{\pi})] \gtrsim \left(\frac{H^2 Cd}{n}\right)^{\frac{1}{2 \rho}}.
    \end{align*}
    \label{theorem: lower bound for MDPs}
\end{theorem}

\subsection{Upper bounds} 
\label{section:MDP-upper-bound}

We now establish upper bounds for this setting. Due to space limitations, we only state the main result here and defer the details (including concrete algorithms) to \Cref{subsection: upper bound for MDPs}. 

\begin{theorem}
    Let $\delta > 0$. Assume that $|\gA| = K$. Then, there exists a learning algorithm that for %
    any problem instance in the set $\gM(\rho,C,d)$, 
    given offline data $S$ of size $n$, returns a policy $\hat{\pi}$ such that with probability at least $1 - \delta$ over the   randomness of generating $S$, we have 
    \begin{align*}
        \sE \left[\subopt_M^\pi(\hat{\pi}) | S \right] \widetilde{<} ~ H \left(\frac{ H^2 C_\pi (Kd + \log(1/\delta)) }{n} \right)^{\frac{1}{2 \rho_\pi}}, 
    \end{align*}
    relative to any comparator $\pi \in \Pi$. 
    \label{theorem: upper bound in MDPs}
\end{theorem}

There is a gap of $\gO(H)$ between our upper and lower bounds, something that can potentially be improved using variance-aware algorithms. 
Nonetheless, \Cref{theorem: upper bound in MDPs} improves and generalizes the results of \citet{xie2021bellman} and \citet{nguyen-tang2023on} on several  fronts. First, since policy transfer coefficients subsume other notions of data coverage, our results hold for a larger class of problem instances. 
Second, our results hold for any function class with finite $L_1$ covering number. In contrast, the guarantees in \citet{xie2021bellman} hold only for finite function classes and  in \citet{nguyen-tang2023on} for function classes with finite domain-wide $L_{\infty}$ covering numbers. 
Third, even when specializing to the setting of \citet{xie2021bellman,nguyen-tang2023on} (i.e., $\rho_\pi = 1$ and finite function class), our bounds decay as $n^{-1/2}$ which is optimal whereas
their bounds scale as $n^{-1/4}$ in the worst-case. 
    Finally, we provide a lower bound for general function approximation. %

\section{Discussion}
\label{section: discussion}
{We study the statistical complexity of offline decision-making with value function approximation. We identify a large class of offline learning problems, characterized by the pseudo-dimension of the value function class and a new characterization of the offline data, for which we provide tight minimax lower and upper bounds. We also provide insights into the role of offline data for online decision-making from a minimax perspective.
}

{We remark that our results do not imply that pseudo-dimension and policy transfer coefficients are necessary conditions for learnability in offline decision-making with function approximation. Consequently, there are several notable gaps in our current understanding. {First}, our lower bounds apply only to parametric function classes, i.e., for function classes with finite pseudo-dimension. Whereas our upper bounds hold also for non-parametric function classes. Can we provide a lower bound for general function classes, including non-parametric ones? {Second}, in the hybrid setting with a parametric function class, the upper bound of our adaptive algorithm matches the lower bound only in the regime that $m \geq \left( \frac{n}{C^* K d} \right)^{1/\rho^*}$. Is it possible to design a fully adaptive algorithm (i.e., without the knowledge of policy transfer coefficients) whose upper bound matches the lower bound in any regime of $m$, including the practical scenarios where the online exploration budget $m$ is small? {Third}, what are the necessary conditions for the value function class and the data quality that \emph{fully} characterize the statistical complexity of offline decision-making with value function approximation?}

\section*{Impact Statement}
This paper presents work whose goal is to advance the field of Machine Learning. There are many potential societal consequences of our work, none which we feel must be specifically highlighted here.

\section*{Acknowledgements}
This research was supported, in part, by DARPA GARD award HR00112020004 and NSF CAREER award IIS-1943251.

\bibliographystyle{icml2024}
\bibliography{main.bib}

\newpage
\appendix
\onecolumn

\section{Uniform Bernstein's Inequality}
\label{section: uniform Bernstein's inequality}
A central tool for our analysis is the uniform concentrability of a subset of functions near the ERM (empirical risk minimizer) predictor. In order to be able to match the lower bounds of offline decision-making (at least) for parametric classes, we require that every ball centered around the ERM predictor within a sufficiently small radius (in fact, of the order of $\gO(\frac{1}{n})$) has an $\gO(\frac{1}{n})$ order in the excess risk. While fast rates are in general not possible, they are so under certain conditions \citep{van2015fast}, including the bounded, squared loss in parametric classes that we consider. A uniform Bernstein's inequality is also presented in \citep[Theorem~B]{cucker2002mathematical}, but using a strong notion of domain-wide $L_{\infty}$ covering numbers. \citep[Theorem~3.21]{zhang2023mathematical} presents a version of uniform Bernstein's inequality using the population $L_1$ covering numbers. The population $L_1$ however requires the knowledge of the data distribution. Here, we present a more practical version of uniform Bernstein's inequality using the empirical $L_1$ covering numbers, which is, at least in principle, computable given the empirical data. More importantly, and also as a key technical result in this section, we prove in \Cref{theorem: Uniform Bernstein's inequality} a uniform Bernstein's inequality for Bellman-like loss functions using the empirical $L_1$ covering numbers. \Cref{theorem: Uniform Bernstein's inequality} applies to handle the data structure generated by RL and, as a special case, naturally applies to contextual bandits. These results are central to our analysis tool and might be of independent interest. Notably, our proof for \Cref{theorem: Uniform Bernstein's inequality} relies on an elegant argument of localizing a function class into a set of balls centered around the covering functions, and then performing uniform convergence in each of such local balls before combining them via a union bound. This localization argument is inspired by a similar argument by \citep{mehta2014stochastic}. 

Before stating the uniform Bernstein's inequality for Bellman-like loss functions in \Cref{theorem: Uniform Bernstein's inequality}, we start with the uniform Bernstein's inequality for a simpler case, which we also use in our analysis and also serves a good point for demonstrating the localization argument.

\subsection{Uniform Bernstein’s inequality for generic case}
\begin{proposition}[Uniform Bernstein's inequality for generic case]
Let $\gG$ be a set of functions $g: \gZ \rightarrow [-b,b]$. Fix $n \in \sN$. Denote $\hat{\sE} g := \frac{1}{n} \sum_{i=1}^n g(z_i)$ where $\{z_i\}_{i} \overset{i.i.d.}{\sim} P$ and $\sE g = \sE_{z \sim P}[g(z)]$, and $\sV[g]$ is the variance of $g(z)$. Then, for any $\delta > 0$, with probability at least $1 - \delta$, we have 
    \begin{align*}
        \forall g \in \gG: \sE g - \hat{\sE} g \leq \inf_{\eps > 0} \left\{\sqrt{ \frac{2 \sV[g] \log(2N_1(\gG, \eps, n)/\delta)}{n}} + \frac{62 b \log(6N_1(\gG, \eps, n)/\delta)}{n}  + 61 \eps \right\}.
    \end{align*}
    \label{prop: uniform Bernstein inequality}
\end{proposition}

Fix $\eps > 0$ and let $N = N_1(\gG, \eps, n)$. Let $\{g_i\}_{i \in [N]}$ be an $\eps$-cover of $\gG$ w.r.t. $L_1(\hat{P}_n)$. For any $i \in [N]$, denote $\gG_i = \{g \in \gG: \hat{\sE}|g - g_i| \leq \eps \}$. We have $\gG \subseteq \cup_{i \in [N]} \gG_i$. The proof of \Cref{prop: uniform Bernstein inequality} relies on the following lemma. 
\begin{lemma}
    Fix any $i \in [N]$. With probability at least $1 - \delta$, 
    \begin{align*}
        \sup_{g \in \gG_i}\sE |g - g_i| \leq 12 \eps + \frac{12 b \log(3/\delta)}{n}. 
    \end{align*}
    \label{lemma: bound E|g-gi|}
\end{lemma}
\begin{proof}[Proof of \Cref{lemma: bound E|g-gi|}] The key idea to obtain fast rates in the above bounds is leveraging the \emph{non-negativity} of the random objects of interest by employing \Cref{lemma: fast rates with non-negative functions}. With probability at least $1 - \delta$, for any $g \in \gG_i$, we have 
    \begin{align*}
        \sE|g - g_i| &\leq 4 \hat{\sE} |g - g_i| + \inf_{\eps'} \left\{ 8 \eps' + \frac{12b \log(3 N_1(\gG_i, \eps', n)/\delta)}{n} \right\}\\ 
        &\leq 12 \eps + \frac{12 b \log(3/\delta)}{n},
    \end{align*}
    where the second inequality follows from that $\hat{\sE}|g-g_i| \leq \eps$, that we choose $\eps'=\eps$, and that the $\eps$-ball $\gG_i$ can be $\eps$-covered by one point, thus $N_1( \gG_i, \eps_i, n) = 1$. 
\end{proof}

\begin{proof}[Proof of \Cref{prop: uniform Bernstein inequality}]
Denote the event $E = E_1 \cap E_2$, where
\begin{align*}
    E_1 &= \left\{ \forall i \in [N], \sE g_i - \hat{\sE} g_i \leq \frac{2b \log(N/\delta)}{3n} + \sqrt{ \frac{2 \sV[g_i] \log(N/\delta)}{n}} \right\} \\ 
    E_2 &= \left\{ \sup_{i \in [N]}\sup_{g \in \gG_i}\sE |g - g_i| \leq 12 \eps + \frac{12 b \log(3N/\delta)}{n}\right\}. 
\end{align*}
Now consider under the event $E$. Since $\gG \subseteq \cup_{i \in [N]} \gG_i$, for any $g \in \gG$, there must exist $i \in [N]$ such that $g \in \gG_i$ (consequently, $\hat{\sE}|g - g_i|\leq \epsilon$). Also notice that 
\begin{align*}
    \sV[g] - \sV[g_i] &= \sE[g^2] - \sE[g_i^2] + \sE[g_i]^2 - \sE[g]^2 \leq 2b \sE |g - g_i| + 2 b |\sE[g_i] - \sE[g]| \leq 4b \sE|g-g_i|. 
\end{align*}
Thus, we have
    \begin{align*}
        \sE g - \hat{\sE} g &= \sE g_i - \hat{\sE} g_i + \sE (g - g_i) + \hat{\sE} (g_i - g) \\ 
        &\leq \sE g_i - \hat{\sE} g_i + \sE |g - g_i| + \hat{\sE} |g - g_i| \\ 
        &\leq \frac{2 b \log(N/\delta)}{3n} + \sqrt{ \frac{2 \sV[g_i] \log(N/\delta)}{n}} + \sE |g - g_i| + \hat{\sE} |g - g_i| \\ 
        &\leq \frac{2b \log(N/\delta)}{3n} + \sqrt{ \frac{2 \sV[g] \log(N/\delta)}{n}} +  \sqrt{ \frac{2 |\sV[g] - \sV[g_i]| \log(N/\delta)}{n}} + \sE |g - g_i| + \hat{\sE} |g - g_i| \\ 
        &\leq \frac{2b \log(N/\delta)}{3n} + \sqrt{ \frac{2 \sV[g] \log(N/\delta)}{n}} +   \frac{b\log(N/\delta)}{n} +  \frac{1}{b}|\sV[g] - \sV[g_i]| + \sE |g - g_i| + \hat{\sE} |g - g_i| \\ 
        &\leq \frac{5 b \log(N/\delta)}{3n} + \sqrt{ \frac{2 \sV[g] \log(N/\delta)}{n}} +  5 \sE |g - g_i| + \hat{\sE} |g - g_i| \\ 
        &\leq  \frac{5 b \log(N/\delta)}{3n} + \sqrt{ \frac{2 \sV[g] \log(N/\delta)}{n}} +    5(12 \eps + \frac{12 b \log(3N/\delta)}{n}) + \eps \\ 
        &\leq \frac{62 b \log(3N/\delta)}{n}  + \sqrt{ \frac{2 \sV[g] \log(N/\delta)}{n}} + 61 \eps,
    \end{align*}
    \revise{where the fourth inequality follows AM-GM inequality}. 
    Finally, note that, by Bernstein's inequality and the union bound, $\pr(E_1) \geq 1 - \delta$; by \Cref{lemma: bound E|g-gi|} and the union bound, $\pr(E_2) \geq 1 - \delta$. Thus, by the union bound again, $\pr(E) \geq 1 - 2 \delta$. 
\end{proof}
\subsection{Uniform Bernstein’s inequality for Bellman-like loss classes}
We now establish the uniform Bernstein's inequality for the Bellman-like loss functions. The nature of the result and the proof is similar to those for \Cref{prop: uniform Bernstein inequality}, but only more involved as we deal with a more structural random tuple. 

Consider a tuple of random variables $(x,a,r,x') \in \gX \times \gA \times [0,1] \times \gX$ distributed according to distribution $P$.  
For any $u: \gX \times \gA \rightarrow \sR$ and $g: \gX \times \gA \rightarrow \sR$, we define the random variable
\begin{align*}
    M(u,g) =  (u(x,a) - r - g(x') )^2 - ( g^*(x,a) - r - g(x') )^2, 
\end{align*}
where $g^*(x,a) := \sE_{(r,x') \sim P(\cdot| x,a)} \left[ r + g(x')\right]$. Let $\{(x_t,a_t,r_t,x'_t)\}_{t \in [n]}$ be an i.i.d. sample from $P$. We write $M_t$ in replace of $M$ when we replace $(x,a,r,x')$ in $M$ by $(x_t,a_t,r_t,x'_t)$. We consider function classes $\gU \subseteq \{u: \gX \times \gA \rightarrow [0,b]\}$ and $\gG \subseteq \{g: \gX \rightarrow [0,b]\}$. We assume, for simplicity, that $r + g(x') \in [0,b]$ almost surely.

\begin{proposition}[Uniform Bernstein's inequality for Bellman-like loss functions]
    Fix any $\eps > 0$. With probability at least $1 - \delta$, for any $u \in \gU, g \in \gG$,  
    \begin{align*}
        \sE [(u(x,a) - g^*(x,a))^2] &\leq \frac{2}{n} \sum_{t=1}^n M_t(u, g) + \inf_{\eps > 0} \left\{ 108 b \eps + b^2 \frac{36 \log N_1(\gU, \eps,n) + 83 \log N_1(\gG, \eps,n) + 108 \log(12/\delta)}{n} \right\}. 
    \end{align*}

    In addition, with probability at least $1 - \delta$, for any $u \in \gU, g \in \gG$,
    \begin{align*}
         -\frac{1}{n} \sum_{t=1}^n M_t(u, g) \leq \inf_{\eps > 0} \left\{ 30 b \eps + b^2\frac{4 \log N_1(\gU, \eps,n)  + 28 \log N_1(\gG, \eps,n) + 28 \log(6/\delta)}{n} \right\}.
    \end{align*}
    \label{theorem: Uniform Bernstein's inequality}
\end{proposition}
\begin{remark}
    \citet{krishnamurthy2019active} develop a uniform Freedman-type inequality for an indicator-type martingale difference sequence, which does not require the data-generating policy to be non-adaptive as in our case. Applying the uniform Freedman-type inequality of \citet{krishnamurthy2019active} ``as-is'' to the stochastic contextual bandit with a non-adaptive data-generating policy results in larger constants and an additional factor of $\log(K)$ (see \citep{foster2018practical} for why this additional $\log K$ is the case). However, there are deeper reasons why our uniform Bernstein's inequality is preferred over the uniform Freedman-type inequality of \citet{krishnamurthy2019active}. First, it seems impossible even if we want to apply the uniform Freedman-type inequality of \citet{krishnamurthy2019active} ``as-is'' to the RL setting. The reason is that the martingale structure in the uniform Freedman-type inequality of \citet{krishnamurthy2019active} is formed only through the adaptive policy that selects actions based on the historical data, and they require i.i.d. assumption on $\{x_t\}_{t \geq 1}$. In contrast, in RL, even when we assume that the initial state is i.i.d. from a fixed distribution, the states from $h \geq 2$ are not independent once the policy that generates the offline data is adaptive. Second, the uniform Freedman-type inequality of \citet{krishnamurthy2019active} cannot seem to easily apply to the RL setting considered in \Cref{prop: uniform Bernstein inequality}, where 
    we consider multiple target regression functions $g^*$, whereas they require one fixed target function. 
    \label{remark: compare uniform Bernstein's inequality to Krish's uniform Freedman-type inequality}

\end{remark}

Fix $\eps > 0$. For simplicity, we write $N_g = N_1(\gG, \eps,n)$ and $N_u = N_1(\gU, \eps,n)$. 
\begin{lemma}
    For any $u \in \gU$ and any $g \in \gG$, we have 
    \begin{align*}
        \sE[M(u,g)] &= \sE[(u(x,a) - g^*(x,a))^2], \\
        \sV[M(u,g)] &\leq 4b^2 \sE[M(u,g)]. 
    \end{align*}
    \label{lemma: variance condtion of Bellman}
\end{lemma}
Let $\{u_i\}_{i \in N_u(\eps)}$ and $\{g_j\}_{j \in N_g(\eps)}$ be the corresponding $\eps$-cover of $\gU$ and $\gG$. Let $\gU_i$ be the set of functions $u \in \gU$ such that $ \|u - u_i\|_{L_1(\{(x_t,a_t)\}_{t \in [n]})} \leq \eps$. Similarly, we define $\gG_j$. 

The following lemma establishes a finite function class variant of \Cref{theorem: Uniform Bernstein's inequality}, over (a finite number of) the covering functions of $\gU$ and $\gG$. 
\begin{lemma}
    With probability $1 - \delta$, for any $(i,j) \in [N_u] \times [N_g]$, we have 
\begin{align*}
    \sE [(u_i(x,a) - g^*_j(x,a))^2] &\leq \frac{2}{n} \sum_{t=1}^n M_t(u_i, g_j) + \frac{12 b^2\log(N_u N_g/\delta)}{n}.
\end{align*}
In addition, with probability $1 - \delta$, for any $(i,j) \in [N_u] \times [N_g]$, we have
\begin{align*}
     -\frac{1}{n} \sum_{t=1}^n M_t(u_i, g_j) &\leq \frac{4 b^2\log(N_u N_g/\delta)}{n}.
\end{align*}
\label{lemma: Freedman's inequality in covering functions}
\end{lemma}
\begin{proof}[Proof of \Cref{lemma: Freedman's inequality in covering functions}]
    For any $(i,j) \in [N_u] \times [N_g]$, by the Freedman's inequality, with probability at least $1 - \delta$, for any $t \in [0,\frac{1}{b^2}]$
\begin{align*}
    \sE [(u_i(x,a) - g^*_j(x,a))^2] &\leq \frac{1}{n} \sum_{t=1}^n M_t(u_i, g_j) +  (e-2) \lambda \sV[M(u_i,g_j)] + \frac{\log(1/\delta)}{\lambda n} \\ 
    &\leq \frac{1}{n} \sum_{t=1}^n M_t(u_i, g_j) +  4 b^2 (e-2) \lambda \sE M(u_i,g_j) + \frac{\log(1/\delta)}{\lambda n},
\end{align*}
where the second inequality follows from \Cref{lemma: variance condtion of Bellman} (the second inequality). By choosing $\lambda = \frac{1}{8(e-2)b^2}$, the above inequality becomes: 
\begin{align*}
    \sE [(u_i(x,a) - g^*_j(x,a))^2] &\leq \frac{2}{n} \sum_{t=1}^n M_t(u_i, g_j) + \frac{12 \log(1/\delta)}{n}. 
\end{align*}
Taking the union bound over $(i,j) \in [N_u] \times [N_g]$ completes the proof for the first part. 

The second part follows similarly except the only difference is that we choose $\lambda = \frac{1}{4(e-2) b^2}$ in the Freedman's inequality. 
\end{proof}

The following lemma shows that a ball in the $L_1$ distance is provably contained within a ball with an empirical $L_1$ distance with a slightly larger radius by a margin that scales at ``fast rates'' $\tilde{O}(\frac{1}{n})$.  
\begin{lemma}
\begin{enumerate}
    \item Fix $i \in [N_u]$. With probability at least $1 - \delta$, 
    \begin{align*}
        \sup_{u \in \gU_i}\sE |u - u_i| \leq 12 \eps + \frac{12 b \log(3/\delta)}{n}.
    \end{align*}
    \item Fix $j \in [N_g]$. With probability at least $1 - \delta$, 
    \begin{align*}
        \sup_{g \in \gG_j}\sE |g - g_j| \leq 12 \eps + \frac{12 b \log(3/\delta)}{n}.
    \end{align*}
    \item Fix $j \in [N_g]$. With probability at least $1 - \delta$, we have
    \begin{align*}
        \sup_{g \in \gG_j}\frac{1}{n} \sum_{t=1}^n \sE_{x' \sim P(\cdot|x_t,a_t)} |g-g_j|(x') \leq 12 \eps + \frac{12 b \log(3/\delta)}{n}.
    \end{align*}
\end{enumerate}
\label{lemma: from population balls to empirical ones}
\end{lemma}
\begin{proof}[Proof of \Cref{lemma: from population balls to empirical ones}] The nature of these results are similar to \Cref{lemma: bound E|g-gi|} in the generic case, where the key idea to obtain fast rates in the estimation errors is leveraging the \emph{non-negativity} of the random objects of interest by employing \Cref{lemma: fast rates with non-negative functions}.
We start with Part 1. 
    With probability at least $1 - \delta$, for any $u \in \gU_i$, we have 
    \begin{align*}
        \sE|u - u_i| &\leq 4 \hat{\sE} |u - u_i| + \inf_{\eps'} \left\{ 8 \eps' + \frac{12b \log(3 N_1(\eps', \gU_i, n)/\delta)}{n} \right\} \leq 12 \eps + \frac{12 b \log(3/\delta)}{n}. 
    \end{align*}
    where the second inequality follows from that $\hat{\sE}|u-u_i| \leq \eps$, that we choose $\eps'=\eps$, and that the $\eps$-ball $\gU_i$ can be $\eps$-covered by one point, thus $N_1(\eps, \gU_i, n) = 1$.

    Part 2 follows exactly as Part 1. For Part 3, we notice 
    \begin{align*}
        \sE \left[ \frac{1}{n} \sum_{t=1}^n |g - g_j|(x'_t) | \{(x_t,a_t)\}_{t \in [n]} \right] =\frac{1}{n} \sum_{t=1}^n \sE_{x' \sim P(\cdot|x_t,a_t)} |g-g_j|(x').
    \end{align*}
    Thus, the same arguments in Part 1 and Part 2 apply to Part 3. 
\end{proof}

We are now ready to prove \Cref{theorem: Uniform Bernstein's inequality}. 
\begin{proof}[Proof of \Cref{theorem: Uniform Bernstein's inequality}]
We start with the following simple discretization: For any $u,u', g, g'$, we have
\begin{align}
    \label{eq: discretize M(u,g)}
    M(u,g) - M(u',g') &\leq 2b |u - u'|(x,a) + 4b|g - g'|(x') + 2b \sE_{x' \sim P(\cdot|x,a)}|g - g'|(x'), \\
    \label{eq: discretize u - g_*}
    (u - g_*)^2(x,a) - (u'-g'_*)^2(x,a) &\leq 2b|u - u'|(x,a) + 2b \sE_{x' \sim P(\cdot|x,a)}|g - g'|(x').
\end{align}

Consider the following event: 
\begin{align*}
     E_1 = &\left\{\sup_{i \in [N_u]} \sup_{u \in \gU_i}\sE |u - u_i| \leq 12 \eps + \frac{12 b \log(3 N_u/\delta)}{n} \right\} \cap \\
     &\left\{\sup_{j \in [N_g]} \sup_{g \in \gG_j}\sE |g - g_j| \leq 12 \eps + \frac{12 b \log(3 N_g/\delta)}{n} \right\} \cap  \\
     &\left\{\sup_{j \in [N_g]} \sup_{g \in \gG_j} \hat{\sE} \sE_{x' \sim P(\cdot|x,a)}|g - g_j|(x')  \leq 12 \eps + \frac{12 b \log(3 N_g/\delta)}{n} \right\} \cap \\
     &\left\{ \sE [(u_i(x,a) - g^*_j(x,a))^2] \leq \frac{2}{n} \sum_{t=1}^n M_t(u_i, g_j) + \frac{12 b^2\log(N_u N_g/\delta)}{n} \right\}. 
\end{align*}
Note that for any $u \in \gU$ and $g \in \gG$, there exist $i \in [N_u]$ and $j \in N_g$ such that $u \in \gU_i$ and $g \in \gG_j$. Thus, under event $E_1$, for any $u \in \gU, g\in \gG$, we have
\begin{align*}
    \sE[(u(x,a) - g_*(x,a))^2] &\leq \sE[(u_i(x,a) - g^*_j(x,a))^2] + 2b \sE |u - u_i| + 2b \sE |g - g_j| \\
    &\leq  \frac{2}{n} \sum_{t=1}^n M_t(u_i, g_j) + \frac{12 b^2\log(N_u N_g/\delta)}{n} + 2b \sE |u - u_i| + 2b \sE |g - g_j| \\ 
    &\leq \frac{2}{n} \sum_{t=1}^n M_t(u, g) + 4b \hat{\sE}|u-u'| + 8b \hat{\sE}|g-g'| + 4b \hat{\sE} \sE_{x' \sim P(\cdot|x,a)}|g - g'|(x') \\
    &+  \frac{12 b^2\log(N_u N_g/\delta)}{n} + 2b \sE |u - u_i| + 2b \sE |g - g_j| \\ 
    &\leq  \frac{2}{n} \sum_{t=1}^n M_t(u, g) + 4b \eps + 8b \eps + 4b( 12 \eps + \frac{12 b \log(3 N_g/\delta)}{n}) \\ 
    &+ \frac{12 b^2\log(N_u N_g/\delta)}{n} + 2b ( 12 \eps + \frac{12 b \log(3 N_u/\delta)}{n}) + 2b( 12 \eps + \frac{12 b \log(3 N_g/\delta)}{n}) \\ 
    &\leq \frac{2}{n} \sum_{t=1}^n M_t(u, g) + 108 b \eps + b^2 \frac{36 \log(N_u) + 83 \log(N_g) + 108 \log(3/\delta)}{n},
\end{align*}
where the first inequality follows from \Cref{eq: discretize u - g_*}, the second inequality follows from $E_1$, the third inequality follows from Equation~(\ref{eq: discretize M(u,g)}), and the fourth inequality follows from $E_1$. By \Cref{lemma: from population balls to empirical ones}, \Cref{lemma: Freedman's inequality in covering functions} and the union bound, we have $\pr(E_1) \geq 1 - 4\delta$. Thus, we complete the first part of the theorem. 

For the second part, the proof follows similarly. In particular, we consider the following event: 
\begin{align*}
     E_2 = &\left\{\sup_{j \in [N_g]} \sup_{g \in \gG_j} \hat{\sE} \sE_{x' \sim P(\cdot|x,a)}|g - g_j|(x')  \leq 12 \eps + \frac{12 b \log(3 N_g/\delta)}{n} \right\} \cap \\
     &\left\{  -\frac{1}{n} \sum_{t=1}^n M_t(u_i, g_j) \leq \frac{4 b^2\log(N_u N_g/\delta)}{n} \right\}. 
\end{align*}
It follows from \Cref{lemma: Freedman's inequality in covering functions} (second part), \Cref{lemma: from population balls to empirical ones}, and a union bound, that $\Pr(E_2) \geq 1 - 2 \delta$.

Finally, note that under event $E_2$, for any $u \in \gU, g \in \gG$, we have
\begin{align*}
    -\frac{1}{n} \sum_{t=1}^n M_t(u, g) &\leq -\frac{1}{n} \sum_{t=1}^n M_t(u_i, g_j) + 2 b \hat{\sE} |u - u_i| + 4b \hat{\sE}|g - g_j| + 2 b \hat{\sE} \sE_{x' \sim P(\cdot|x,a)}|g - g_j|(x') \\ 
    &\leq \frac{4 b^2\log(N_u N_g/\delta)}{n} + 2b \eps + 4b \eps + 2b (12 \eps + \frac{12 b \log(3 N_g/\delta)}{n}) \\ 
    &\leq 30 b \eps + b^2\frac{4 \log(N_u) + 28 \log( N_g) + 28 \log(3/\delta)}{n}, 
\end{align*}
where the first inequality follows from Eq.~(\ref{eq: discretize M(u,g)}), the second inequality follows from the conditions in event $E_2$.

\end{proof}

\section{Proofs of \Cref{section: lower bounds for offline CB}}
\label{section: proofs of lower bounds for offline CB}

\begin{proof}[Proof of \Cref{theorem: lower bound for offline CB}]
Fix any $(C, \rho, d)$ as stated in the theorem. Let $\eps = \left(\frac{Cd}{32n}\right)^{1/(2 \rho)}$. We have $0 < \eps < 1/2$ and $\frac{\eps^{2 \rho -2 }}{C} \leq 1$. \\

\paragraph{Construction of hard instances.}
Let $\gA = \{a_1, a_2\}$. Pick any $d$ mutually distinct points $x_1, \ldots, x_d$. We construct a family of the context-reward distributions $\gD_{\sigma}$ indexed by $\sigma \in \{-1,1\}^d$ where $\gD_{\sigma}(x,y) = P_X(x) \times P^{\sigma}_{Y|X}(y) $, $P_X(x_i) = \frac{1}{d}$, $P^{\sigma}_{Y|X}(y(a_1)|x_i) = \Ber(\frac{1}{2})$, and $P^{\sigma}_{Y|X}(y(a_2)|x_i) = \Ber(\frac{1}{2} + \sigma_i \frac{\eps}{2})$. Choose any $\mu$ such that $\mu(a_2|x_i) = \frac{\epsilon^{2\rho -2}}{C} \leq 1, \forall i \in [d]$. Now we choose the function class $\gF$ as follows: $\gF_{a_1} = \{f(x) \equiv \frac{1}{2}\}$ and $\gF_{a_2} = \{f_{\alpha}(x_i) := \frac{1}{2} + \alpha_i \frac{\eps}{2}, \forall i \in [d] | \alpha \in \{-1,1\}^d\}$.\\ 

\paragraph{Verification.} We now verify that for any $\sigma$, $(\gD_\sigma, \mu, \pi^*_{
\gD_{\sigma}}, \gF) \in \gB(\rho,C,d)$. First, it is easy to see that $\pdim (\gF_{a_2}) = d$ and $\pdim (\gF_{a_1}) = 0$.
Second, we have $f^*_{\sigma}(x_i, a_1) = \frac{1}{2}$ and $f^*_{\sigma}(x_i,a_2) = \frac{1}{2} + \sigma_i \frac{\eps}{2}$, where $f^*_{\sigma}(x,a) := \sE_{P_\sigma}[Y(a) |X=x]$, thus $f^*_\sigma \in \gF, \forall \sigma$. Third, for any $f \in \gF$, we have $f(\cdot, a_2) = f_{\alpha}$ for some $\alpha \in \{-1,1\}^d$ and for any policy $\pi$, we have 
\begin{align*}
    \gE_{{\sigma}, \pi}(f) = \eps^2 \sum_{i=1}^d \pi(a_2|x_i) \frac{\mathbbm{1}\{\sigma_i \neq \alpha_i\}}{d},
\end{align*}
where $\gE_{{\sigma}, \pi}(f)$ denotes the excess risk of $f$ under $\gD_{\sigma} \otimes \pi$, i.e., under the squared loss and realizability, $\gE_{{\sigma}, \pi}(f) = \sE_{\gD_\sigma \otimes \pi} [(f - f^*)^2]$.

Since we choose $\mu(a_2|x_i) = \frac{\eps^{2\rho -2}}{C} \leq 1, \forall i \in [d]$, we have 
\begin{align*}
    \gE_{{\sigma}, \mu}(f) = \frac{\eps^{2\rho}}{C} \frac{\textrm{dist}(\sigma, \alpha)}{d} \geq \frac{1}{C} \left(\eps^2 \frac{\textrm{dist}(\sigma, \alpha)}{d} \right)^{\rho} \geq \frac{1}{C} \gE^{\rho}_{{\sigma}, \pi^*}(f),
\end{align*}
since $\textrm{dist}(\sigma, \alpha) \leq d$ and $\rho \geq 1$. 
Thus, we have $(\gD_{\sigma}, \mu, \pi^*_{P_{\sigma}}, \gF) \in \gB(\rho,C,d)$ for any $\sigma \in \{-1,1\}^d$. \\

\paragraph{Reduction to testing.}
For any policy $\pi$, we have $V^{\pi}_{\sigma} := V^{\pi}_{\gD_\sigma} = \sum_{i=1}^d \frac{1}{d} \left( \pi(a_1|x_i)(\frac{1}{2} ) + \pi(a_2|x_i)(\frac{1}{2} + \sigma_i \frac{\eps}{2})\right)$. Thus, we can compute the sub-optimality of the output policy $\hat{\pi}$ of any offline learner by
\begin{align*}
    V^*_{\sigma} - V_{\sigma}^{\hat{\pi}} = \sum_{i=1}^d \frac{\eps}{2d} \left(\mathbbm{1}\{\sigma_i\} - \hat{\pi}(a_2|x_i) \sigma_i \right).
\end{align*}
Let $\hat{\sigma}_i = \mathbbm{1}\{\hat{\pi}(a_2 | x_i) \geq \frac{1}{2} \}$. We have $\mathbbm{1}\{\sigma_i\} - \hat{\pi}(a_2|x_i) \sigma_i \geq \frac{|\sigma_i - \hat{\sigma}_i|}{4}$, and thus, we have
\begin{align}
    \sE_{\sigma} \left[ V^*_{\sigma} - V_{\sigma}^{\pi} \right] \geq \frac{\eps}{4d} \sE_{\sigma} \left[ \textrm{dist}(\sigma, \hat{\sigma}) \right],
    \label{eq: sub-optimality in terms of eps}
\end{align}
where $\textrm{dist}(\sigma, \hat{\sigma}) := \sum_{i=1}^d \mathbbm{1}\{\sigma_i \neq \hat{\sigma}_i\}$ denotes the Hamming distance between two binary vectors $\sigma$ and $\hat{\sigma}$, and $\sE_\sigma$ denotes the expectation over the randomness of $\hat{\sigma}$ with respect to $P_\sigma$. 

 The worst-case Hamming distance $\sup_{\sigma \in \{-1,1\}^d}\sE_{\sigma} \left[ \textrm{dist}(\sigma, \hat{\sigma}) \right]$ can be lower-bounded using the standard tools in hypothesis testing: 
\begin{align}
    \sup_{\sigma \in \{-1,1\}^d}\sE_{\sigma} \left[ \textrm{dist}(\sigma, \hat{\sigma}) \right] &\geq \frac{d}{2} \min_{\sigma, \sigma': \textrm{dist}(\sigma,\sigma')=1} \inf_{\psi} \left[ \gD_\sigma(\psi \neq \sigma) + \gD_{\sigma'}(\psi \neq \sigma')\right] \nonumber\\ 
    &\geq \frac{d}{2} \left( 1 - \sqrt{\frac{1}{2} \max_{\sigma, \sigma': \textrm{dist}(\sigma,\sigma')=1} \kl\left( (\gD_\sigma \otimes \mu)^n \| (\gD_{\sigma'} \otimes \mu)^n \right)}\right),
    \label{eq: reduction to testing}
\end{align}
where the first inequality follows Assouad's lemma \citep[Lemma~2.12]{tsybakov1997nonparametric} and the second inequality follows from \citep[Theorem~2.12]{tsybakov1997nonparametric}.

We now compute the KL distance: For any $\sigma$ and $\sigma'$ such that $\textrm{dist}(\sigma, \sigma') = 1$, let $i^* \in [d]$ be the (only) coordinate that $\sigma$ differs from $\sigma'$, we have
\begin{align*}
    &\kl \left((\gD_{\sigma} \otimes \mu)^n \| (\gD_{\sigma'} \otimes \mu)^n \right) = n \kl \left( \gD_{\sigma} \otimes \mu \| \gD_{\sigma'} \otimes \mu \right) \\ 
    &= n \sE_{P_X} \kl \left(\mu \otimes P_{\sigma}(Y|X) \| \mu \otimes P_{\sigma'}(Y|X) \right) \\ 
    &= \frac{n}{d} \sum_{i=1}^d \kl \left( \mu(a_1|x_i) \Ber(\frac{1}{2}) + \mu(a_2|x_i) \Ber(\frac{1}{2} + \sigma_i \frac{\eps}{2}) \big\| \mu(a_1|x_i) \Ber(\frac{1}{2}) + \mu(a_2|x_i) \Ber(\frac{1}{2} + \sigma'_i \frac{\eps}{2})  \right) \\ 
    &\leq \frac{n}{d} \sum_{i=1}^d\mu(a_2 | x_i) \kl \left( \Ber(\frac{1}{2} + \sigma_i \frac{\eps}{2}) \big\| \Ber(\frac{1}{2} + \sigma'_i \frac{\eps}{2})\right) \\ 
    &= \frac{n}{d} \mu(a_2 | x_{i^*}) \kl \left( \Ber(\frac{1}{2} + \sigma_{i^*} \frac{\eps}{2}) \big\| \Ber(\frac{1}{2} + \sigma'_{i^*} \frac{\eps}{2})\right) \\ 
    &\leq 16 \frac{n}{d} \mu(a_2 | x_{i^*}) \eps^2 = \frac{16 n \eps^{2 \rho}}{Cd} = \frac{1}{2},
\end{align*}
where the first inequality uses the convexity of KL divergence, the second inequality uses a basic KL upper bound that $\kl\left( \Ber\left(\frac{1}{2} + z \frac{\eps}{2} \right) \big\| \Ber\left(\frac{1}{2} - z \frac{\eps}{2} \right) \right) \leq 16 \eps^2$ for any $z \in \{-1,1\}$, and the second-last equality plugs in the choice of $\mu$, and the last equality plugs in the choice of $\eps$. Now, plugging the above inequality into \Cref{eq: reduction to testing} and \Cref{eq: sub-optimality in terms of eps}, we have 
\begin{align*}
    \max_{\sigma \in \{-1,1\}^d}\sE_{\sigma} \left[ V^*_{\sigma} - V_{\sigma}^{\pi} \right] \geq \frac{\eps}{16} = \frac{1}{16}\left(\frac{Cd}{32n}\right)^{1/(2 \rho)}.
\end{align*}
\end{proof}

\section{Proofs of \Cref{section: upper bound ofr CB offline}}

We restate \Cref{theorem: upper bounds of OfDM-Hedge CB} in an exact form with disclosed constants. 
\begin{theorem}
    Fix any $\delta \in [0,1]$. Invoke \Cref{algorithm: OfDM-Hedge} with $\beta =  32  \eps + \frac{4 \log N_1(\gF,\eps,n) + 24 \log(12/\delta)}{n}$. Then, for any $(\gD,\gF)$ such that \Cref{assumption: realizability cb} holds, with probability at least $1 - \delta$, for any $\pi \in \Pi$,
    \begin{align*}
       \sE[\subopt_\gD^\pi(\hat{\pi}) | S] &\leq \left( C_{\rho(\pi)} \inf_{\eps \geq 0} \left\{ 172 \eps + \frac{44 \log N_1(\gF, \eps,n) + 156 \log(24/\delta)}{n} \right \} \right)^{1/(2 \rho(\pi))} + 4 \sqrt{\frac{\log K}{T}} \\ 
    &+ \mathbbm{1}_{\{|\gX| > 1\}} \cdot \inf_{\eps \geq 0} \left\{\sqrt{ \frac{2 \log(2N_1(\gF(\cdot, \Pi), \eps, n)/\delta)}{n}} + \frac{62  \log(6N_1(\gF(\cdot, \Pi), \eps, n)/\delta)}{n}  + 61 \eps \right\}. 
    \end{align*}
    \label{theorem: theorem: upper bounds of OfDM-Hedge CB with exact constants}
\end{theorem}

Our proof relies on the following lemmas.

\begin{lemma}
    Fix any $\eps > 0$, and $\delta \in [0,1]$. Define the version space 
    \begin{align*}
        \gF(\beta) := \{f \in \gF: \hat{P} l_f - \hat{P} l_{\hat{f}} \leq \beta \},
    \end{align*}
    where $\beta =  32  \eps + \frac{4 \log N_1(\gF,\eps,n) + 24 \log(12/\delta)}{n}$. Then, with probability at least $1 - \delta$, $f^* \in \gF(\beta)$, and 
    \begin{align*}
        \sup_{f \in \gF(\beta) }\sE_{\gD \otimes \mu}[(f - f^*)^2] \leq 172 \eps + \frac{44 \log N_1(\gF, \eps,n) + 156 \log(24/\delta)}{n}. 
    \end{align*}
    \label{lemma: Uniform Bernstein's inequality for CB structure}
\end{lemma}
\begin{proof}[Proof of \Cref{lemma: Uniform Bernstein's inequality for CB structure}]
    \Cref{lemma: Uniform Bernstein's inequality for CB structure} is an instantiation of \Cref{theorem: Uniform Bernstein's inequality} for the contextual bandit case where we have a tuple of random variables $(x,a,r)$ instead of having an additional transition variable $x'$. In the contextual bandit case, we simply use $\gG =\{0\}$ the set of the zero function, and remove all the terms regarding $N_1(\gG, \eps,n)$ (which is $1$ in this case) in the RHS of \Cref{theorem: Uniform Bernstein's inequality}. 
\end{proof}

\begin{lemma}
    Choose $\eta = \sqrt{\frac{\log K}{4(e-2)T}}$ and $T \geq \frac{\log K}{e-2}$. Then we have
    \begin{align*}
        \forall \pi, \sum_{t=1}^T \hat{P}(f_t(\cdot,\pi) - f_t(\cdot,\pi_t)) \leq 4\sqrt{T \log K}. 
    \end{align*}
    \label{lemma: exponential weights update}
\end{lemma}
\begin{proof}[Proof of \Cref{lemma: exponential weights update}]
It suffices to show a stronger result, that any $x \in \gX$ and $\pi \in \Pi$,  
\begin{align*}
    \sum_{t=1}^T f_t(x,\pi) - f_t(x, \pi_t) \leq 4\sqrt{T \log K}. 
\end{align*}
This is a standard guarantee of the Hedge algorithm, where a complete proof can be found at \citep[Lemma~B.6]{nguyen-tang2023on}. 
\end{proof}

\begin{proof}[Proof of \Cref{theorem: theorem: upper bounds of OfDM-Hedge CB with exact constants}]
    \begin{align*}
    V^{\pi} - V^{\pi_t} &= P f^*(\cdot, \pi) - P f^*(\cdot,\pi_t) \\
    &=\sE_{\gD \otimes \pi}[f^* - f_t]
    + (P - \hat{P})f_t(\cdot,\pi)
    + \hat{P}(f_t(\cdot,\pi) - f_t(\cdot,\pi_t))
    +  (\hat{P} - P) f^*(\cdot,\pi_t) +\hat{P} (f_t(\cdot,\pi_t) - f^*(\cdot,\pi_t)),  \\ 
    &\leq \underbrace{\sE_{\gD \otimes \pi}[f^* - f_t]}_{I_1}  + \underbrace{\hat{P}(f_t(\cdot,\pi) - f_t(\cdot,\pi_t))}_{I_2} + 2 \underbrace{\sup_{g \in \gF(\cdot, \Pi)}|(P - \hat{P}) g|}_{I_3} + \underbrace{\hat{P} (f_t(\cdot,\pi_t) - f^*(\cdot,\pi_t))}_{I_4}. 
\end{align*}
We bound each term $I_1, I_2, I_3, I_4$ separately. Regarding term $I_2$, by \Cref{lemma: exponential weights update}, we have 
\begin{align*}
    \sum_{t=1}^T \hat{P}(f_t(\cdot,\pi) - f_t(\cdot,\pi_t)) \leq 4 \sqrt{T \log K}.
\end{align*}

Consider the event $E = E_1 \cap E_2 \cap E_3$, where 
\begin{align*}
    E_1 \!&:= \!\!\left\{\! \sup_{g \in \gF(\cdot, \Pi)} \!\!|(P - \hat{P}) g| \leq  \mathbbm{1}_{\{|\gX| > 1\}}  \inf_{\eps > 0} \left\{\sqrt{ \frac{2 \log(2N_1(\gF(\cdot, \Pi), \eps, n)/\delta)}{n}} + \frac{62  \log(6N_1(\gF(\cdot, \Pi), \eps, n)/\delta)}{n}  + 61 \eps \right\} \right\}, \\ 
    E_2 &:= \left\{ f^* \in \gF(\beta)\right\}, \\
    E_3 &:= \left\{ \sup_{f \in \gF(\beta) }\sE_{\gD \otimes \mu}[(f - f^*)^2] \leq 172 \eps + \frac{44 \log N_1(\gF, \eps,n) + 156 \log(24/\delta)}{n} \right\}. 
\end{align*}

Due to pessimism of \Cref{algorithm: OfDM-Hedge}, $f_t = \displaystyle\argmin_{f \in \gF: \hat{P} l_f - \hat{P} l_{\hat{f}} \leq \beta} \hat{P} f(\cdot, \pi_t)$. Thus, under $E_2$, $I_4 = \hat{P} f_t(\cdot, \pi_t) - \hat{P} f^*(\cdot,\pi_t) \leq 0$. By the policy transfer definition, we have 
\begin{align*}
    I_1 &= \sE_{\gD \otimes \pi} [f^* - f_t] \leq \left( C_\pi \sE_{\gD \otimes \mu}[(f^* - f)^2] \right)^{1/(2 \rho_\pi)}. 
\end{align*}
Thus, under event $E$, we have 
\begin{align*}
    V^\pi - \frac{1}{T} \sum_{t=1}^T V^{\pi_t} &\leq \left( C_\pi \left( 172 \eps + \frac{44 \log N_1(\gF, \eps,n) + 156 \log(24/\delta)}{n} \right) \right)^{1/(2 \rho_\pi)} + 4 \sqrt{\frac{\log K}{T}} \\ 
    &+ \mathbbm{1}\{|\gX| > 1\} \land \inf_{\eps > 0} \left\{\sqrt{ \frac{2 \log(2N_1(\gF(\cdot, \Pi), \eps, n)/\delta)}{n}} + \frac{62  \log(6N_1(\gF(\cdot, \Pi), \eps, n)/\delta)}{n}  + 61 \eps \right\}. 
\end{align*}

Now we bound the probability of each of the events $E_1, E_2, E_3$ and combine them via the union bound for the probability of $E$. Note that if $|\gX| = 1$, we have a trivial inequality $I_3 = 0$. 
Combining with \Cref{prop: uniform Bernstein inequality}, we have $\pr(E_1) \geq 1 - \delta$. By \Cref{lemma: Uniform Bernstein's inequality for CB structure}, we have $\pr(E_2 \cap E_3) \geq 1 - \delta$. Thus, $\pr(E) \geq 1 - 2 \delta$.

\end{proof}

\section{Proofs for \Cref{section: extension to hybrid setting}}

\subsection{The expected sub-optimality guarantees for FALCON+ \citep[Algorithm~2]{simchi2022bypassing}}
\label{subsection: digest for FALCON+}
When we employ FALCON+ \citep[Algorithm~2]{simchi2022bypassing} for the online learning step in \Cref{algorithm: hybrid learner}, the returned policy $\hat{\pi}_{\textrm{on}}$ has the expected regret over $m/2$ rounds of order $\sqrt{K \gE_{\gF,\delta/\log(m/2)}(m) } m$, where $\gE_{\gF,\delta}(n)$  is a learning guarantee for the offline regression oracle (see their Assumption 2). 
\begin{assumption}[{\citep[Assumption~2]{simchi2022bypassing}}]
    Let $\pi$ be an arbitrary policy. Given $n$ training samples of the form $(x_i, a_i, r_i(a_i))$ i.i.d. drawn according to $(x_i,r_i) \sim \gD$ and $a_i \sim \pi(\cdot|x_i)$, the offline regression oracle returns a predictor $\hat{f}: \gX \times \gA \rightarrow \sR$. For any $\delta > 0$, with probability at least $1 - \delta$, we have 
    \begin{align*}
        \sE_{x \sim \gD_\gX, a \sim \pi(\cdot|x)} \left[(\hat{f}(x,a) - f^*(x,a))^2 \right] \leq \gE_{\gF,\delta}(n). 
    \end{align*}
\end{assumption}
\revise{Note that, the output for running FALCON++ for $m/2$ iterates is a sequence of $m/2$ policies. $\hat{\pi}_{\textrm{on}}$ is an uniform distribution over such $m/2$ iterates.}

Since the context $x$ is i.i.d. from $\gD_{\gX}$, we can transform the expected regret $\sqrt{K \gE_{\gF,\delta/\log(m/2)}(m) } m$ into the expected sub-optimality bound of order $\sqrt{K \gE_{\gF,\delta/\log(m/2)}(m) } + \frac{1}{m}$ using an (improved) online-to-batch argument (e.g., \citep{nguyentang2022instancedependent}), wherein $\frac{1}{m}$ is the cost of converting a regret bound into a sub-optimality bound. It is known that if the closure of $\gF$ is convex (which we assume here for simplicity of comparison), the sample complexity of learning $\gF$ using squared loss is $\frac{\pdim(\gF)}{\eps}(\log(1/\eps) + \log(1/\delta))$ \citep{lee1996importance}. Thus,  $\gE_{\gF,\delta}(n) \leq \frac{\pdim(\gF)}{n} \log(n/\delta)$. Finally, note that $\pdim(\gF) \leq K d$.

\subsection{Proof of \Cref{theorem: lower bound for the hybrid setting}}
\begin{proof}[Proof of \Cref{theorem: lower bound for the hybrid setting}]
Construct the hard problem instances exactly like the ones in the proof of \Cref{theorem: lower bound for offline CB}, except that we now choose 
\begin{align*}
    \eps = \min\left\{\left(\frac{Cd}{64n}\right)^{1/(2 \rho)}, \frac{1}{8}\sqrt{\frac{d}{m}} \right\}. 
\end{align*}
The verification and the reduction to testing follow exactly, except for the step in which we compute the KL divergence of the observations between two different models. Specifically, denote $Q_\sigma$ as the probability of the pre-collected data and the online data under the model $P_\sigma$. Let $S_{\textrm{on}} = \{\tilde{x}_i, \tilde{a}_i, \tilde{r}_i\}_{i \in [m]}$ be the random online data collected during the online phase by a hybrid algorithm $\textrm{ALG}$. Note that $\tilde{a}_t$ depends on $\tilde{x}_t, \{\tilde{x}_i, \tilde{a}_i, \tilde{r}_i\}_{i \in [t-1]}$, and $S_{\textrm{off}}$ for any $t \in [m]$. We denote this conditional distribution by $P_{\textrm{ALG}}(\tilde{a}_t | \tilde{x}_t, \{\tilde{x}_i, \tilde{a}_i, \tilde{r}_i\}_{i \in [t-1]}, S_{\textrm{off}})$. Note that the conditional distribution $P_{\textrm{ALG}}$ depends only on the algorithm $\textrm{ALG}$ and invariant to any underlying model $P_\sigma$, thus we have 
\begin{align*}
    & \frac{ Q_{\sigma}( S_{\textrm{on}})}{ Q_{\sigma'}( S_{\textrm{on}})} =  \prod_{i \in [m]}\frac{P_\sigma(\tilde{r}_i | \tilde{x}_i, \tilde{a}_i)}{P_{\sigma'}(\tilde{r}_i | \tilde{x}_i, \tilde{a}_i)}.
\end{align*}
Consider any $\sigma$ and $\sigma'$ such that $\textrm{dist}(\sigma, \sigma') = 1$. We thus have
\begin{align*}
    \sum_{S_{\textrm{off}} \cup S_{\textrm{on}}}Q_{\sigma}(S_{\textrm{on}}) \log \frac{ Q_{\sigma}( S_{\textrm{on}})}{Q_{\sigma'}( S_{\textrm{on}})} 
    &=  \sum_{S_{\textrm{off}} \cup S_{\textrm{on}}} Q_{\sigma}(S_{\textrm{on}}) \prod_{i \in [m]}\frac{P_\sigma(\tilde{r}_i | \tilde{x}_i, \tilde{a}_i)}{P_{\sigma'}(\tilde{r}_i | \tilde{x}_i, \tilde{a}_i)}\\
    &\leq  \frac{m}{d}\sum_{i=1}^d \kl \left( \Ber(\frac{1}{2} + \sigma_i \frac{\eps}{2}) \big\| \Ber(\frac{1}{2} + \sigma'_i \frac{\eps}{2})\right)  \leq \frac{m}{d} 16 \eps^2.
\end{align*}
Hence, we have
\begin{align*}
    \kl(Q_{\sigma} \| Q_{\sigma'}) &= \sum_{S_{\textrm{off}} \cup S_{\textrm{on}}}Q_{\sigma}(S_{\textrm{off}} \cup S_{\textrm{on}})  \log \frac{ Q_{\sigma}(S_{\textrm{off}} \cup S_{\textrm{on}})}{ Q_{\sigma'}(S_{\textrm{off}} \cup S_{\textrm{on}})} \\ 
    &= \sum_{S_{\textrm{off}}}Q_{\sigma}(S_{\textrm{off}}) \log \frac{ Q_{\sigma}(S_{\textrm{off}})}{ Q_{\sigma'}(S_{\textrm{off}})} + \sum_{S_{\textrm{off}} \cup S_{\textrm{on}}}Q_{\sigma}(S_{\textrm{on}}) \log \frac{ Q_{\sigma}( S_{\textrm{on}})}{Q_{\sigma'}( S_{\textrm{on}})} \\ 
    &\leq \frac{16 n \eps^{2\rho}}{Cd} + \frac{16 m \eps^2}{d} \leq 1/4 + 1/4 = 1/2,
\end{align*}
where the first term in the first inequality follows from the upper bound of KL of two distributions of the offline data in the proof of \Cref{theorem: lower bound for offline CB}, the second term of the first inequality follows from a direct calculation, and the last inequality follows from the choice of $\eps$ presented at the beginning of the proof.

 The worst-case Hamming distance $\sup_{\sigma \in \{-1,1\}^d}\sE_{\sigma} \left[ \textrm{dist}(\sigma, \hat{\sigma}) \right]$ can be lower-bounded using the standard tools in hypothesis testing: 
\begin{align*}
    \sup_{\sigma \in \{-1,1\}^d}\sE_{\sigma} \left[ \textrm{dist}(\sigma, \hat{\sigma}) \right] &\geq \frac{d}{2} \min_{\sigma, \sigma': \textrm{dist}(\sigma,\sigma')=1} \inf_{\psi} \left[ Q_\sigma(\psi \neq \sigma) + Q_{\sigma'}(\psi \neq \sigma')\right] \nonumber\\ 
    &\geq \frac{d}{2} \left( 1 - \sqrt{\frac{1}{2} \max_{\sigma, \sigma': \textrm{dist}(\sigma,\sigma')=1} \kl\left( Q_\sigma \| Q_{\sigma'} \right)}\right) \\ 
    &\geq \frac{d}{4},
\end{align*}
where the first inequality follows Assouad's lemma \citep[Lemma~2.12]{tsybakov1997nonparametric} and the second inequality follows from \citep[Theorem~2.12]{tsybakov1997nonparametric}. 

Let $\hat{\sigma}_i = \mathbbm{1}\{\hat{\pi}(a_2 | x_i) \geq \frac{1}{2} \}$. We have $\mathbbm{1}\{\sigma_i\} - \hat{\pi}(a_2|x_i) \sigma_i \geq \frac{|\sigma_i - \hat{\sigma}_i|}{4}$, and thus, we have
\begin{align*}
     \sup_{\sigma \in \{-1,1\}^d} \sE_{\sigma} \left[ V^*_{\sigma} - V_{\sigma}^{\pi} \right] \geq \frac{\eps}{4d}  \sup_{\sigma \in \{-1,1\}^d} \sE_{\sigma} \left[ \textrm{dist}(\sigma, \hat{\sigma}) \right] \geq \frac{\eps}{16} = \frac{1}{16}  \min\left\{\left(\frac{Cd}{64n}\right)^{1/(2 \rho)}, \frac{1}{8}\sqrt{\frac{d}{m}} \right\}. 
\end{align*}

\end{proof}

\section{Proofs for \Cref{section: extension to MDPs}}

\subsection{Proof of \Cref{theorem: lower bound for MDPs}}

\begin{proof}[Proof of \Cref{theorem: lower bound for MDPs}]
The proof structure follows similarly as that of \Cref{theorem: lower bound for offline CB}. 
\paragraph{Construction of a family of hard MDPs.} Each MDP $M_\sigma$ is characterized by $\sigma \in \{-1,1\}^d$. For any $\sigma$, $M_\sigma$ is a deterministic MDP with the state space is $\gS = \{x_1, \bar{x}_1, \tilde{x}_1, \ldots, x_d, \bar{x}_d, \tilde{x}_d\}$, the action space $\gA = \{a_1, a_2\}$ (see \Cref{fig:hard MDPs}). Each MDP $M_\sigma$ starts uniformly at one of $d$ states $x_1, \ldots, x_d$. Form each state $x_i$ for any $i \in [d]$, one can follow the ``blue'' path by taking action $a_1$ or the ``red'' path by taking action $a_2$. This always lead to an absorbing state ($\bar{x}_i$ in the blue path and $\tilde{x}_i$ in the red path). Taking any action from an absorbing state leads to the same state. If one starts from $x_i$ for any $i\in[d]$, the reward for every blue arrow in the graph (resp. red arrow) is $\frac{1}{2}$ (resp. $\frac{1}{2} + \sigma_i \frac{\eps}{2}$). Also note that all the MDPs in the family share the same (deterministic) dynamics (they are only different by reward labelling). 
\begin{figure}[h!]
    \centering
    \includegraphics[scale=0.5]{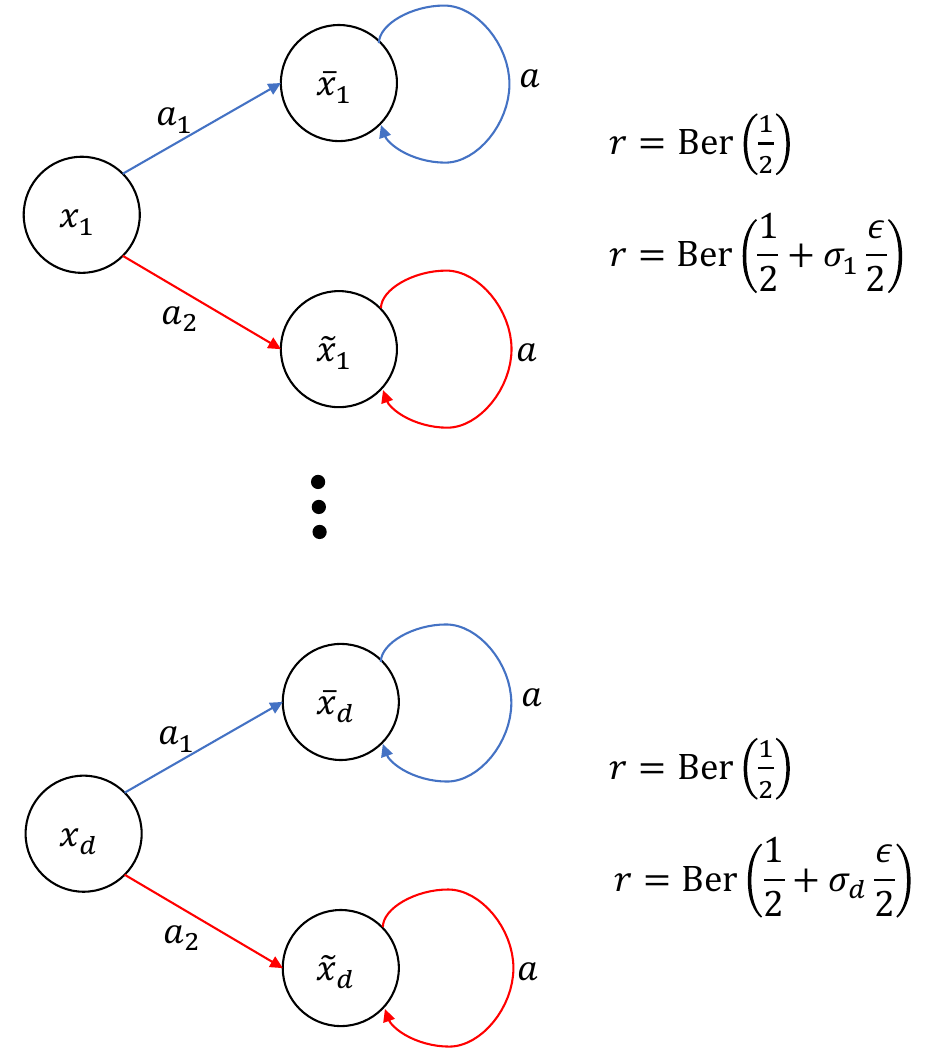}
    \caption{Hard MDPs}
    \label{fig:hard MDPs}
\end{figure}

We can compute exactly Q-functions of each policy under each MDP $M_\sigma$. Note that starting from $h \geq 2$, the value function does not depend on the action being taking. The total reward in any trajectory essentially depends on which initial state one starts with and which action one takes from the initial state. Thus under any MDP $M_\sigma$, its Q-value functions under a policy $\pi$ have the property that they are completely agnostic to $\pi$ -- which comes handy in satisfying the value realizability and Bellman completeness. This property is inspired by the construction by \cite{foster2021offline}, though our constructions are different and much simpler.

Specifically, for any $\pi$, we have 
\begin{align*}
    Q^{\pi}_h(\bar{x}_i, a) &= \frac{H - h +1}{2}, \forall h \geq 1, a \in \gA, i \in [d] \\ 
    Q^{\pi}_h(\tilde{x}_i, a) &= (H + 1 -h) (\frac{1}{2} + \sigma_i \frac{\eps}{2}), \forall h \geq 1, a \in \gA, i \in [d] \\ 
    Q_1^{\pi}(x_i, a_1) &= \frac{H}{2}, \forall i \in [d] \\ 
    Q_1^{\pi}(x_i, a_2) &= H (\frac{1}{2} + \sigma_i \frac{\eps}{2}), \forall i \in [d].
\end{align*}

We construction the following function class $\gF = \{f^{\sigma}: \sigma \in \{-1,1\}^d\}$ where 
\begin{align*}
    f_1^{\sigma}(x_i, a_1) &= \frac{H}{2}, ~~~~~~~ f^{\sigma}_1(x_i, a_2) = H(\frac{1}{2} + \sigma_i \frac{\eps}{2}), \\
    \forall h \in [2,H], f^{\sigma}_h(\bar{x}_i, a) &= (H + 1 -h) \frac{1}{2}, ~~~~~~~ f^{\sigma}_h(\tilde{x}_i, a) = (H + 1 -h) (\frac{1}{2} + \sigma_i \frac{\eps}{2}). 
\end{align*}
It is easy to see that for any $\sigma \in \{-1,1\}^d$, the pair $(\gF, M_\sigma)$ satisfies the value realizability and the Bellman completness. The value realizability follows from that $\gF$ is constructed as the bare minimum function class that contains all possible Q-value functions of the MDPs in the class. The Bellman completeness follows from the value realizability and that the Bellman backup under any policy $\pi$ on a function in $\gF$ does not depend on $\pi$. 

We have 
\begin{align*}
    V_1^{\pi} = \frac{H}{d} \sum_{i=1}^d \left( \pi_1(a_1|x_i) \frac{1}{2} + \pi_1(a_2|x_i) (\frac{1}{2} + \sigma_i \frac{\eps}{2}) \right).
\end{align*}
Thus, we have 
\begin{align}
    \sE_{\sigma} \left[ V^*_{\sigma} - V_{\sigma}^{\pi} \right] \geq \frac{H \eps}{4d} \sE_{\sigma} \left[ \textrm{dist}(\sigma, \hat{\sigma}) \right],
    \label{eq: sub-optimality in terms of eps for MDPs}
\end{align}
For any $f \in \gF$, we have $f = f^{\alpha}$ for some $\alpha \in \{-1,1\}^d$. For any policy $\pi$, we have 
\begin{align*}
    \forall h \in [H], \sE_{\pi} [(f_h^\sigma - f_h^\alpha)^2] = \eps^2 (H - 1 + 1)^2\sum_{i=1}^d \pi_1(a_2|x_i) \frac{\mathbbm{1}\{\sigma_i \neq \alpha_i\}}{d}. 
\end{align*}
Choose 
\begin{align*}
    \mu(a_2 | x_i) = \frac{\eps^{2\rho -2 } H^{2 \rho -2}}{C}.
\end{align*}
This is valid only when $ C \geq (\eps H)^{2\rho -2}$. Then we have 
\begin{align*}
    \forall h \in [H], C |\sE_{\pi} [f_h^\sigma - f_h^\alpha]|^2 \leq C \sE_{\pi} [(f_h^\sigma - f_h^\alpha)^2] \leq \sE_{\mu} [(f_h^\sigma - f_h^\alpha)^2]^\rho. 
\end{align*}

The offline data can be equivalently reduced into $S_n = \{(s_1^t, a_1^t, r_1^t)\}_{t \in [n]}$ because the information in the first timestep $h=1$ fully captures the information in subsequent time steps $h \geq 1$.  For any $\sigma$ and $\sigma'$ such that $\textrm{dist}(\sigma, \sigma') = 1$, let $i^* \in [d]$ be the (only) coordinate that $\sigma$ differs from $\sigma'$, we have
\begin{align*}
    \kl\left(\pr_\sigma(S_n) \|  \pr_{\sigma'} (S_n)\right) \leq 16 \frac{n}{d} \mu(a_2 | x_{i^*}) \eps^2 \leq \frac{16 n H^{2 \rho - 2} \eps^{2 \rho}}{Cd} = \frac{1}{2},
\end{align*}
where we choose 
\begin{align*}
    \eps = \left( \frac{Cd}{32 n H^{2 \rho - 2}} \right)^{\frac{1}{2 \rho}}. 
\end{align*}

The worst-case Hamming distance $\sup_{\sigma \in \{-1,1\}^d}\sE_{\sigma} \left[ \textrm{dist}(\sigma, \hat{\sigma}) \right]$ can be lower-bounded using the standard tools in hypothesis testing: 
\begin{align*}
    \sup_{\sigma \in \{-1,1\}^d}\sE_{\sigma} \left[ \textrm{dist}(\sigma, \hat{\sigma}) \right] &\geq \frac{d}{2} \min_{\sigma, \sigma': \textrm{dist}(\sigma,\sigma')=1} \inf_{\psi} \left[ \Pr_\sigma(\psi \neq \sigma) + \Pr_{\sigma'}(\psi \neq \sigma')\right] \nonumber\\ 
    &\geq \frac{d}{2} \left( 1 - \sqrt{\frac{1}{2} \max_{\sigma, \sigma': \textrm{dist}(\sigma,\sigma')=1} \kl\left( \Pr_\sigma(S_n) \| \Pr_{\sigma'}(S_n) \right)}\right) \\ 
    &\geq \frac{d}{4},
\end{align*}
where the first inequality follows Assouad's lemma \citep[Lemma~2.12]{tsybakov1997nonparametric} and the second inequality follows from \citep[Theorem~2.12]{tsybakov1997nonparametric}. 

Plugging into \Cref{eq: sub-optimality in terms of eps for MDPs}, we have 
\begin{align*}
      \sup_{\sigma \in \{-1,1\}^d} \sE_{\sigma} \left[ V^*_{\sigma} - V_{\sigma}^{\pi} \right] &\geq \frac{H \eps}{4d}  \sup_{\sigma \in \{-1,1\}^d} \sE_{\sigma} \left[ \textrm{dist}(\sigma, \hat{\sigma}) \right] \\
      &\geq \frac{H \eps}{16} = \frac{1}{16} \left( \frac{H^2 C d}{ 32 n} \right)^{\frac{1}{2 \rho}}. 
\end{align*}
\end{proof}

\subsection{Proof of \Cref{theorem: upper bound in MDPs}}
\label{subsection: upper bound for MDPs}
We first re-state \Cref{theorem: upper bound in MDPs}. 

\begin{theorem}
    Let \Cref{assumption: realizability for mdp}, \Cref{assumption: Bellman completeness}, and \Cref{assumption: parametric class for MDP} hold and assume that $|\gA| = K$. There exists a (possibly randomized) learning algorithm $\hat{\pi}$ such that for any MDP $M$, and any $\delta \in [0,1]$, with probability at least $1 - \delta$, for any $\pi$ such that $\rho_\pi \geq 0.5$, 
    \begin{align*}
        \sE \left[\subopt_M^\pi(\hat{\pi}) | S \right] = C_\pi ^{\frac{1}{2 \rho_\pi}} H^{1 - \frac{1}{2 \rho_\pi}}  \left(H^2 \eps + H^3 \frac{d(\eps, n) + \log(H/\delta)}{n}  \right)^{\frac{1}{2 \rho_\pi}},
    \end{align*}
    where $d(\eps,n) := \max_{h \in [H]}\{\log N_1(\gF_h, \eps, n) \lor  \log N_1(\gF_h(\cdot, \Pi_h), \eps, n)\}$. 
    \label{theorem: upper bound in MDPs restated}
\end{theorem}

We provide a specific algorithm, \Cref{algorithm: OfDM-Hedge-MDP}, that obtains the bounds in \Cref{theorem: upper bound in MDPs restated}. \Cref{algorithm: OfDM-Hedge-MDP} is essentially the OfDM-Hedge algorithm (\Cref{algorithm: OfDM-Hedge}) extended to MDPs. Note that \Cref{algorithm: OfDM-Hedge-MDP} also already appears in the prior works of \cite{xie2021bellman,nguyen-tang2023on}.

\begin{algorithm}
   \caption{Hedge for Offline Decision-Making in MDP (OfDM-Hedge-MDP)}
\begin{algorithmic}[1]
   \STATE {\bfseries Input:} Offline data $S$, function class $\gF$
   \STATE {\bfseries Hyperparameters:} Confidence parameter $\beta$, learning rate $\eta$, iteration number $T$
   \STATE Initialize $\pi^{(1)} = \{ \pi^{(1)}_h \}_{h \in [H]}$, where $\pi^{(1)}_h(\cdot|x) = \textrm{Uniform}(\gA)$, $\forall x \in \gX_h$
   \FOR{$t=1$ {\bfseries to} $T$}
   \STATE Pessimism: $f^{(t)} = \displaystyle\argmin_{f \in \gF(\beta, \pi^{(t)})}  f_1(s_1, \pi_1^{(t)})$ where 
   \begin{align*}
       \gF(\beta, \pi^{(t)}) := \left\{f \in \gF: \sum_{i \in [n]} \sum_{h \in [H]}  L_i(f_h, f_{h+1},\pi^{(t)}) - \inf_{g \in \gF} \sum_{i \in [n]} \sum_{h \in [H]}  L_i(g_h, f_{h+1},\pi^{(t)}) \leq \beta \right\}
   \end{align*}
   \STATE Hedge: $\pi^{(t+1)}_h(a|x) \propto \pi_h^{(t)}(a|x) e^{\eta f_h^{(t)}(x,a)}, \forall (x,a, h)$
   \ENDFOR
   \STATE {\bfseries Output:} A randomized policy $\hat{\pi}$ as a uniform distribution over $\{\pi^{(t)}\}_{t \in [T]}$.
   \label{algorithm: OfDM-Hedge-MDP}
\end{algorithmic}
\end{algorithm}

\paragraph{Notations.} For convenience, we denote the element-wise functionals indexed by functions and policies: 
\begin{align*}
    L_i(f_h, f_{h+1},\pi) &:= (f_h(x_h^{(i)},a_h^{(i)} ) - r_h^{(i)} - f_{h+1}(x^{(i)}_{h+1}, \pi_{h+1}))^2, \\ 
    Z_i(f_h, f_{h+1},\pi) &:= L_i(f_h, f_{h+1},\pi) -  L_i(T^{\pi}_h f_{h+1} , f_{h+1},\pi),\\ 
    \gE_h^{{\pi}}(f_h, f_{h+1})(x,a) &:= (T^{\pi}_h f_h - f_{h+1})(x,a).
\end{align*}

A key starting point for the proof of \Cref{theorem: upper bound in MDPs restated} of our upper bounds in this section is the error decomposition lemma that relies on a notion of induced MDPs, used originally in \citep{zanette2021provable} and adopted in \citep{nguyen-tang2023on}. 

\begin{definition}[Induced MDPs]
For any policy $\pi$ and any sequence of functions $Q = \{Q_h\}_{h \in [H]} \in \{\gX \times \gA \rightarrow \sR\}^H$, the $(Q,\pi)$-induced MDPs, denoted by $M(Q,\pi)$ is the MDP that is identical to the original MDP $M$ except only that the expected reward of $M(Q,\pi)$ is given by $\{r_h^{\pi,Q}\}_{h \in [H]}$, where 
\begin{align*}
    r_h^{\pi,Q}(x,a) := r_h(x,a) - (T^{\pi}_h f_h - f_{h+1})(x, a).
\end{align*}
\label{defn: induced mdp}
\end{definition}

By definition of $M(\pi,Q)$, $Q$ is the fixed point of the Bellman equation $Q_h = T^{\pi}_{h,M(\pi,Q)} Q_{h+1}$. 
\begin{lemma}
    For any policy $\pi$ and any sequence of functions $Q = \{Q_h\}_{h \in [H]} \in \{\gX \times \gA \rightarrow \sR\}^H$, we have 
\begin{align*}
    Q^{\pi}_{M(\pi, Q)} = Q,
\end{align*}
where $M(\pi, Q)$ is the induced MDP given in \Cref{defn: induced mdp}. 
\label{lemma: exact value function under the induced MDP}
\end{lemma}

A key lemma that we use is the following error decomposition. 
\begin{lemma}[Error decomposition]
    For any action-value functions $Q \in \{\gS \times \gA \rightarrow \sR\}^H$ and any policies $\pi, \tilde{\pi} \in \Pi$, we have 
    \begin{align*}
        \subopt^M_{\pi}(\tilde{\pi}) = \sum_{h=1}^H \sE_{\pi} [\gE_h^{\tilde{\pi}}(Q_h, Q_{h+1})(x_h,a_h)] + Q_1(x_1, \tilde{\pi}_1) - V_{1,M}^{\tilde{\pi}}(x_1) + \subopt_{\pi}^{M(Q, \tilde{\pi})}(\tilde{\pi}). 
    \end{align*}
    \label{lemma: error decomposition}
\end{lemma}
\begin{proof}[Proof of \Cref{lemma: error decomposition}]
We have 
\begin{align*}
    &\subopt^M_{\pi}(\tilde{\pi}) = V_1^{\pi}(x_1) - V_1^{\tilde{\pi}}(x_1) \\ 
    &= \left(V_1^{\pi}(x_1) - V_{1,M(Q,\tilde{\pi})}^{\pi}(x_1) \right)+ \left(V_{1,M(Q,\tilde{\pi})}^{\tilde{\pi}}(x_1) - V_1^{\tilde{\pi}}(x_1) \right) + \left(V_{1,M(Q,\tilde{\pi})}^{\pi}(x_1)  - V_{1,M(Q,\tilde{\pi})}^{\tilde{\pi}}(x_1) \right) \\ 
    &=\sum_{h=1}^H \sE_{\pi} [\gE_h^{\tilde{\pi}}(Q_h, Q_{h+1})(x_h,a_h)] + Q_1(x_1, \tilde{\pi}_1) - V_1^{\tilde{\pi}}(x_1) + \subopt_{\pi}^{M(Q, \tilde{\pi})}(\tilde{\pi}), 
\end{align*}
where in the last equality, for the first term, we use, by \Cref{defn: induced mdp}, that 
\begin{align*}
    V_1^{\pi}(x_1) - V_{1,M(Q,\tilde{\pi})}^{\pi}(x_1) &= \sum_{h=1}^H \sE_{\pi} \left[ r_h(x_h, a_h) - r_h^{\tilde{\pi},Q}(x_h, a_h) \right] =\sum_{h=1}^H \sE_{\pi} [\gE_h^{\tilde{\pi}}(Q_h, Q_{h+1})(x_h,a_h)],
\end{align*}
for the second term, we use, by \Cref{lemma: exact value function under the induced MDP}, that 
\begin{align*}
    V_{1,M(Q,\tilde{\pi})}^{\pi}(x_1) = Q_{1,M(Q,\tilde{\pi})}^{\pi}(x_1, \tilde{\pi}_1) = Q_1(x_1, \tilde{\pi}). 
\end{align*}
\end{proof}

\begin{lemma}[\cite{nguyen-tang2023on}]
Consider an arbitrary sequence of value functions $\{Q^t\}_{t \in [T]}$ such that $\max_{h,t}\| Q^t_h \|_{\infty} \leq b$ and define the following sequence of policies $\{\pi^t\}_{t \in [T+1]}$ where 
\begin{align*}
    \pi^1(\cdot|s) &= \unif(\gA), \forall s, \\ 
    \pi^{t+1}_h(a|s) &\propto \pi^t_h(a|s) \exp \left( \eta Q^t_h(s,a) \right), \forall (s,a,h,t). 
\end{align*}
Suppose $\eta = \sqrt{\frac{\ln K}{4(e-2) b^2 T}}$ and $T \geq \frac{\ln K}{(e-2)}$. For an arbitrary policy $\pi \in \Pi$, we have
\begin{align*}
    \sum_{t=1}^T \left( V^{\pi}_{1,M(\pi^t, Q^t)}(x_1) - V_{1,M(\pi^t, Q^t)}^{\pi^t}(x_1) \right) \leq 4H b \sqrt{T \log K}.
\end{align*}
\label{lemma: bounding value difference under soft policy update}
\end{lemma}

\begin{proof}[Proof of \Cref{theorem: upper bound in MDPs restated}]
    By \Cref{lemma: error decomposition}, for every $\pi, t$, we have
    \begin{align*}
        \subopt^M_{\pi}(\pi^{(t)}) = \underbrace{\sum_{h=1}^H \sE_{\pi} [(T^{\pi^{(t)}}_h f^{(t)}_{h+1} - f^{(t)}_h)(s_h,a_h)]}_{A_t} + \underbrace{f^{(t)}_1(s_1, \pi^{(t)}_1) - V_{1,M}^{\pi^{(t)}}(s_1)}_{B_t} + \underbrace{\subopt_{\pi}^{M(f^{(t)}, \pi^{(t)})}(\pi^{(t)})}_{C_t}. 
    \end{align*}
    Thus, we have 
    \begin{align*}
        \sE \left[\subopt_M^\pi(\hat{\pi}) | S \right] = \frac{1}{T} \sum_{t=1}^T A_t +B_t + C_t. 
    \end{align*}
    Note that by \Cref{lemma: bounding value difference under soft policy update}, we have 
    \begin{align*}
        \sum_{t=1}^T C_t \leq 4H^2 \sqrt{T \ln K}. 
    \end{align*}
    Thus, we now only need to bound $\sum_{t=1}^T A_t$ and $\sum_{t=1}^T B_t$. Note that for every $\pi$ and $\tilde{\pi}$, we have
    \begin{align}
    \sum_{h=1}^H \sE_{\pi} \left[ (T^{\tilde{\pi}}_h f_{h+1} - f_h)(s_h,a_h) \right] &\leq \sum_{h=1}^H  C_\pi ^{1/(2 \rho_\pi)} \left(\sE_{\mu} \left[ (T^{\tilde{\pi}}_h f_{h+1} - f_h)(s_h,a_h)^2 \right] \right)^{1/(2 \rho_\pi)} \nonumber \\ 
    &\leq C_\pi ^{1/(2 \rho_\pi)} H^{1 - 1/(2 \rho_\pi)}  \left(\sE_{\mu} \left[ \sum_{h=1}^H (T^{\tilde{\pi}}_h f_{h+1} - f_h)(s_h,a_h)^2 \right] \right)^{1/(2 \rho_\pi)},
    \label{equation: bound bellman error by sum and policy transfer}
\end{align}
where the first inequality follows from the Bellman completeness assumption and the transfer exponent definition, the second inequality follows from Jensen's inequality for a concave function $x \mapsto x^{1/(2 \rho_{\pi})}$ as long as $\rho_{\pi} \geq 1/2$. Thus, to bound $\sum_{t=1}^T A_t$, it suffices to bound the in-distribution squared Bellman error $\sE_{\mu} \left[ \sum_{h=1}^H (T_h^{\pi^{(t)}} f^{(t)}_{h+1} - f^{(t)}_{h})(s_h,a_h)^2 \right]$. This relies on the uniform Bernstein's inequality for Bellman-like loss functions we have developed in \Cref{section: uniform Bernstein's inequality}. 

\begin{lemma}[Uniform Bernstein's inequality for Bellman-like loss functions]
    Fix any $\eps > 0$. With probability at least $1 - \delta$, for any $f \in \gF, \pi \in \Pi$,  
    \begin{align*}
        \sE [(f_h - T^{\pi}_h f_{h+1})^2] &\leq \frac{2}{n} \sum_{t=1}^n Z_t(f_h, f_{h+1},\pi) \\
        &+ \inf_{\eps > 0} \left\{ 108 H \eps + H^2 \frac{36 \log N_1(\gF_h, \eps,n) + 83 \log N_1(\gF_{h+1}(\cdot, \Pi_{h+1}), \eps,n) + 108 \log(12/\delta)}{n} \right\}. 
    \end{align*}

    In addition, with probability at least $1 - \delta$, for any $f \in \gF, \pi \in \Pi$,
    \begin{align*}
         -\frac{1}{n} \sum_{t=1}^n Z_t(f_h, f_{h+1},\pi) \leq \inf_{\eps > 0} \left\{ 32 H \eps + H^2\frac{4 \log N_1(\gF_h, \eps,n)  + 28 \log N_1(\gF_{h+1}(\cdot, \Pi_{h+1}), \eps,n) + 24 \log(6/\delta)}{n} \right\}.
    \end{align*}
    \label{lemma: Uniform Bernstein's inequality for MDP}
\end{lemma}

\begin{proof}[Proof of \Cref{lemma: Uniform Bernstein's inequality for MDP}]
This is a direct application of \Cref{theorem: Uniform Bernstein's inequality}.
    
\end{proof}

Now let's fix $\eps \geq 0$ and $\delta \in [0,1]$. We use $c$ to denote an absolute constant that can change its value at every of its appearance, as we are not interested in absolute constant factors and would like to simplify the presentation. Set $\beta$ in \Cref{algorithm: OfDM-Hedge-MDP} by
\begin{align*}
    \beta = c \left(H^2 \eps + H^3 \frac{d(\eps, n) + \log(H/\delta)}{n} \right). 
\end{align*}
By the second part of \Cref{lemma: Uniform Bernstein's inequality for MDP}, \Cref{assumption: realizability for mdp}, and \Cref{assumption: Bellman completeness}, we have 
\begin{align*}
    \pr \left( \forall \pi, Q^{\pi} \in \gF(\beta, \pi) \right) \geq 1 - \delta.
\end{align*}
Thus, we have 
\begin{align*}
    \pr \left( B_t \leq 0, \forall t \right) \geq 1 - \delta. 
\end{align*}

For bounding $\sum_{t=1}^T A_t$, note that we have 
\begin{align*}
    \sum_{i=1}^n Z_i(f_h^{(t)}, f^{(t)}_{h+1}, \pi^{(t)}) \leq 
    \sum_{i \in [n]} \sum_{h \in [H]}  L_i(f^{(t)}_h, f^{(t)}_{h+1},\pi^{(t)}) - \inf_{g \in \gF} \sum_{i \in [n]} \sum_{h \in [H]}  L_i(g_h, f^{(t)}_{h+1},\pi^{(t)}) \leq \beta,
\end{align*}
where the first inequality follows from \Cref{assumption: Bellman completeness} and the second inequality follows from the design of \Cref{algorithm: OfDM-Hedge-MDP}. Thus, by the first part of \Cref{lemma: Uniform Bernstein's inequality for MDP}, with probability at least $1 - \delta$, we have
    \begin{align*}
        \forall t, \sE \left[ \sum_{h=1}^H (T^{\pi^{(t)}}_h f^{(t)}_{h+1} - f^{(t)}_h)(s_h,a_h)^2 \right] \leq c \left(H^2 \eps + H^3 \frac{d(\eps, n) + \log(H/\delta)}{n} \right).
    \end{align*}
    Plugging the above inequality into the RHS of \Cref{equation: bound bellman error by sum and policy transfer} leads to an upper bound on $\sum_{t=1}^T A_t$. 

\end{proof}

\section{Support Lemmas}

\begin{lemma}[Freedman's inequality]
Let $X_1, \ldots, X_T$ be \emph{any} sequence of real-valued random variables. Denote $\sE_t[\cdot] = \sE[\cdot| X_1, \ldots, X_{t-1}]$. Assume that $X_t \leq R$ for some $R > 0$ and $\sE_t[X_t]=0$ for all $t$. Define the random variables
\begin{align*}
    S := \sum_{t=1}^T X_t, \hspace{10pt} V := \sum_{i=1}^T \sE_t[X_t^2]. 
\end{align*}
Then for any $\delta > 0$, with probability at least $1 - \delta$, for any $\lambda \in [0,1/R]$, 
\begin{align*}
    S \leq (e - 2) \lambda V + \frac{\ln(1/\delta)}{\lambda}.
\end{align*}
\label{lemma: Freedman inequality}
\end{lemma}

The following lemma exploits the non-negativity of the function class to obtain a fast estimation error rate when relating the population quantity to the empirical one. 
\begin{lemma}
    Consider any function class $\gH \subseteq \{\gZ \rightarrow [0,b]\}$ for some $b >0$ and let $S_n = \{z_1,\ldots,z_n\}$ be an i.i.d. sample from a distribution $P \in \Delta(\gZ)$. For any $\delta \in (0,1)$, with probability at least $1 - \delta$ over the randomness of $S_n$, we have
    \begin{align*}
        \forall h \in \gH: P h \leq 4 \hat{P}_n h + \inf_{\eps >0} \left[ 8  \eps + \frac{12 b \ln(3 N_1(\gH, \eps, S_n)/\delta)}{n} \right]. 
    \end{align*}
    \label{lemma: fast rates with non-negative functions}
\end{lemma}
\begin{remark}
    \Cref{lemma: fast rates with non-negative functions} is a generalization of \citep[Theorem~4.12]{zhang2023mathematical} from a function range $[0,1]$ to an arbitrary function range $[0,b]$, i.e. setting $b=1$ in the above lemma reduces to \citep[Theorem~4.12]{zhang2023mathematical}. 
\end{remark}
\begin{proof}[Proof of \Cref{lemma: fast rates with non-negative functions}]
    We start from \citep[Theorem~4.12, \revise{with $\gamma=0.5$ in their theorem}]{zhang2023mathematical} which corresponds to the case $b=1$ of the above lemma. Let $\gH / b := \{h/b: h \in \gH\}$. Since $\|h'\|_{\infty} \leq 1, \forall h' \in \gH/b$, we can apply \citep[Theorem~4.12]{zhang2023mathematical} to $\gH/b$: With probability at least $1 - \delta$: $\forall h \in \gH$, we have
    \begin{align*}
        P\frac{h}{b} &\leq 4 \hat{P}_n \frac{h}{b} + \inf_{\eps >0} \left[ 8  \eps + \frac{12 \ln(3 N_1( \gH/b,\eps, S_n)/\delta)}{n} \right] \\ 
        &\leq 4 \hat{P}_n \frac{h}{b} + \inf_{\eps >0} \left[ 8  \eps + \frac{12 \ln(3 N_1( \gH,\eps b, S_n)/\delta)}{n} \right].
    \end{align*}
    The above inequality implies that 
    \begin{align*}
        P h &\leq 4 \hat{P}_n h + \inf_{\eps >0} \left[ 8  b \eps + \frac{12 b \ln(3 N_1( \gH, \eps b,S_n)/\delta)}{n} \right] \\ 
        &\leq 4 \hat{P}_n h + \inf_{\eps >0} \left[ 8  \eps + \frac{12 b \ln(3 N_1(\gH,\eps,  S_n)/\delta)}{n} \right],
    \end{align*}
    where the last inequality follows from replacing $\eps$ by $\eps/b$ in the first inequality.
\end{proof}

\begin{remark}
    It is possible to obtain a tighter bound specified by the Rademacher complexity of the function class $\gH$ in the above lemma, if we are willing to make an additional assumption that each function in $\gH$ is smooth (and non-negative, which is already satisfied in the above lemma). \revise{The fast rates are achievable via the optimistic rate framework of \cite{srebro2010optimistic}.} The smooth and non-negative condition is satisfied anyway in our case as we use squared loss. However, this result comes at the cost of a large absolute constant in the upper bound. Also, this stronger upper bound ultimately does not benefit our case as our bounds still depend on log-covering numbers, instead of entirely depending on Rademacher complexity. 
\end{remark}

\begin{lemma}
    Let $\gF: \gX \times \gA \rightarrow \sR$, let $\Pi = \{\gX \rightarrow \Delta(\gA)\}$ be the set of all policies. Suppose that $|\gA| = K$. For any $p \geq 1, n \in \sN$, we have
    \begin{align*}
        \max\left\{N_p( \gF, \eps, n), N_p( \gF(\cdot, \Pi),\eps, n) \right\} \leq \prod_{a \in \gA} N_p( \gF(\cdot,a),\frac{\eps}{K^{1/p}}, n). 
    \end{align*}
    \label{lemma: combine coverining numbers to those of factorized classes}
\end{lemma}
\revise{
\begin{proof}[Proof of \Cref{lemma: combine coverining numbers to those of factorized classes}]
    We will first prove that: 
        \begin{align*}
        N_p( \gF(\cdot, \Pi),\eps, n) \leq \prod_{a \in \gA} N_p( \gF(\cdot,a),\frac{\eps}{K^{1/p}}, n). 
    \end{align*}
    The other inequality can be proved similarly. 
    Fix any $\eps > 0, p \geq 1, n \in \sN$. Let $N_a = N_p(\gF(\cdot, a), \eps' ,n), \forall a \in \gA$. For any $g \in \gF(\cdot, \Pi)$, $g = f(\cdot, \pi)$ for some $f \in \gF, \pi \in \Pi$. Let $f'$ such that $\left(\frac{1}{n} \sum_{i=1}^n | f(x_i,a) - f'(x_i,a) |^p \right)^{1/p} \leq \frac{\eps}{K^{1/p}}$ for any $a$. Define $g' = f'(\cdot, \pi)$. We have
    \begin{align*}
        \|g - g'\|_n^p &= \frac{1}{n} \sum_{i=1}^n |f(x_i,\pi) - f'(x_i,\pi)|^p \\ 
        &\leq \frac{1}{n} \sum_{i=1}^n \left(\sE_{a \sim \pi(\cdot|x_i)} \left[|f(x_i,a) - f'(x_i,a)| \right] \right)^p \\ 
        &\leq \frac{1}{n} \sum_{i=1}^n  \sE_{a \sim \pi(\cdot|x_i)}\left[|f(x_i,a) - f'(x_i,a)|^p \right] \\ 
        &\leq \frac{1}{n} \sum_{i=1}^n  \max_{a \in \gA} |f(x_i,a) - f'(x_i,a)|^p \\ 
        &\leq \frac{1}{n} \sum_{i=1}^n  \sum_{a \in \gA} |f(x_i,a) - f'(x_i,a)|^p \\ 
        &\leq \frac{1}{n}  \sum_{a \in \gA} \sum_{i=1}^n  |f(x_i,a) - f'(x_i,a)|^p \leq \epsilon^{p}. 
    \end{align*}
\end{proof}
}

\end{document}